\documentclass{article}

\usepackage[final]{neurips_2022}

\usepackage[utf8]{inputenc} % allow utf-8 input
\usepackage[T1]{fontenc}    % use 8-bit T1 fonts
\usepackage{hyperref}       % hyperlinks
\usepackage{url}            % simple URL typesetting
\usepackage{booktabs}       % professional-quality tables
\usepackage{amsfonts}       % blackboard math symbols
\usepackage{nicefrac}       % compact symbols for 1/2, etc.
\usepackage{microtype}      % microtypography
\usepackage{xcolor}         % colors
\usepackage{amsthm}
\usepackage{graphicx}
\usepackage{xcolor}
\usepackage{wrapfig}
\setcitestyle{numbers,square}
\usepackage{caption}
\usepackage{subcaption}
\usepackage{algorithm}
\usepackage{algpseudocode}
\usepackage{varwidth}% http://ctan.org/pkg/varwidth

\newtheorem{theorem}{Theorem}

\newtheorem{assumption}{Assumption}
\newtheorem{lemma}[theorem]{Lemma}

\newtheorem{remark}{Remark}
\newtheorem{example}{Example}
\usepackage{amsmath}
\newcommand\norm[1]{\left\lVert#1\right\rVert}

\DeclareMathOperator*{\argmin}{arg\,min}

\newcommand{\s}{{\mathcal S}}
\newcommand{\A}{{\mathcal A}}
\newcommand{\C}{{\mathcal C}}
\newcommand{\M}{{\mathcal M}}
\newcommand{\probset}{\mathcal{P}}
\newcommand{\kfunc}{\kappa}
\newcommand{\aaaiparam}{\rho}

\PassOptionsToPackage{numbers, sort}{natbib}
\title{Meta Reinforcement Learning with Finite Training Tasks - a Density Estimation Approach}

\author{%
  Zohar Rimon\\
  Technion - Israel Institute of Technology\\
  \texttt{zohar.rimon@campus.technion.ac.il} \\
   \And
   Aviv Tamar \\
   Technion - Israel Institute of Technology\\
   \texttt{avivt@technion.ac.il}
   \AND
   Gilad Adler \\
   Ford Research Center Israel \\
   \texttt{gadler3@ford.com} \\
}
\usepackage{selectp} 
\begin{document}

\maketitle

\begin{abstract}
  In meta reinforcement learning (meta RL), an agent learns from a set of training tasks how to quickly solve a new task, drawn from the same task distribution. The optimal meta RL policy, a.k.a.~the Bayes-optimal behavior, is well defined, and guarantees optimal reward \textit{in expectation}, taken with respect to the task distribution.
  The question we explore in this work is \textit{how many} training tasks are required to guarantee approximately optimal behavior with high probability. Recent work provided the first such PAC analysis for a model-free setting, where a history-dependent policy was learned from the training tasks. In this work, we propose a different approach: directly learn the task distribution, using density estimation techniques, and then train a policy on the learned task distribution. 
  We show that our approach leads to bounds that depend on the dimension of the task distribution. In particular, in settings where the task distribution lies in a low-dimensional manifold, we extend our analysis to use dimensionality reduction techniques and account for such structure, obtaining significantly better bounds than previous work, which strictly depend on the number of states and actions. 
  The key of our approach is the regularization implied by the kernel density estimation method. We further demonstrate that this regularization is useful in practice, when `plugged in' the state-of-the-art VariBAD meta RL algorithm.
\end{abstract}

\section{Introduction}
In recent years, reinforcement learning (RL) became a dominant algorithmic framework for a variety of domains, including computer games~\cite{DQN}, robotic manipulation~\cite{RoboMan}, and autonomous driving~\cite{DRLCars}. Popular RL algorithms, however, are characterized by a high sample inefficiency, due to the exploration-exploitation problem -- the need to balance between obtaining more information about the environment versus acting based on such
information. Indeed, most RL success stories required a very long training process, only possible in simulation.

% Usually, when such algorithms are being introduced to a new task, their performance depends on how well they balance the exploration-exploitation tradeoff, i.e how well they balance between getting more information about the environment versus acting based on this
% information. While exploring, the agent tries to learn the properties of the environment in order to maximize its reward. The algorithms mentioned above, are usually characterized by a high sample inefficiency and sub-optimal exploration.

For an agent to learn fast, additional structure of the problem is required. In Meta RL~\cite{RL2,VariBad,MAML,rakelly2019efficient}, agents are allowed to train on a set of training tasks, sampled from the same task distribution as the task they will eventually be tested on. The hope is that similar structure between the tasks could be identified during learning, and exploited to quickly solve the test task. 
It has recently been observed that the meta RL problem is related to the Bayesian RL formulation, where each task is modelled as a Markov decision process (MDP), and the distribution over tasks is the \textit{Bayesian prior}~\cite{ortega2019meta,VariBad}. In this formulation, the optimal meta RL policy is the \textit{Bayes-optimal policy} -- the policy that maximizes expected return, where the expectation is taken with respect to the prior MDP distribution.

While significant empirical progress in meta RL has been made, the theoretical understanding of the problem is still limited. A central question, which we focus on in this work, is the probably approximately correct (PAC) analysis of meta RL, namely, \textit{how many training tasks are required to guarantee performance that is approximately Bayes-optimal with high probability}. Practical motivation for studying this question includes the lifelong learning setting~\cite{garcia2019meta}, where training tasks are equivalent to tasks that the agent had previously encountered, and we seek agents that learn as quickly as possible, and the offline meta RL setting~\cite{offlineBRL}, where training data is collected in advance, and therefore estimating how much data is needed is important. 

Recently, \citet{RegBounds} proposed the first PAC bounds for meta RL, using a \textit{model free} approach. In their work, a history-dependent policy was trained to optimize the return on the set of training MDPs, where \textit{policy-regularization} was added to the loss of each MDP in the data. The bounds in \cite{RegBounds} scale with the number of states exponentiated by the length of the history, intuitively due to the number of possible histories that can be input to the policy.
In this work, we propose a \textit{model based} approach for PAC meta RL. Our idea is that instead of regularizing the policy during training, as in \cite{RegBounds}, we can learn a regularized version of the \textit{distribution} of MDPs from the training data. Subsequently, we can train an agent to be Bayes optimal on this learned distribution. Intuitively, if we can guarantee that the learned distribution is `close enough' to the real prior, we should expect the learned policy to be near Bayes optimal. Here, we derive such guarantees using techniques from the kernel density estimation (KDE) literature~\cite{KDE}.

Building on PAC results for KDE, we derive PAC bounds for our model based meta RL approach. Compared to the bounds in \cite{RegBounds}, our results require much less stringent assumptions on the MDP prior, and apply to both continuous and discrete state and action spaces. Interestingly, our bounds have a linear dependence on the horizon, compared to the exponential dependence in \cite{RegBounds}, but have an exponential dependence on the dimension of the prior distribution, corresponding to the exponential dependence on dimensionality of KDE. We argue, however, that in many practical cases the prior distribution lies in some low dimensional subspace. For such cases, we extend our analysis to use dimensionality reduction via principal component analysis (PCA), and obtain bounds where the exponential dependence is on the dimension of the low dimensional subspace.

\begin{wrapfigure}{r}{0.45\textwidth} 
    \centering
    \vspace{-1em}
    \includegraphics[width=0.45\textwidth]{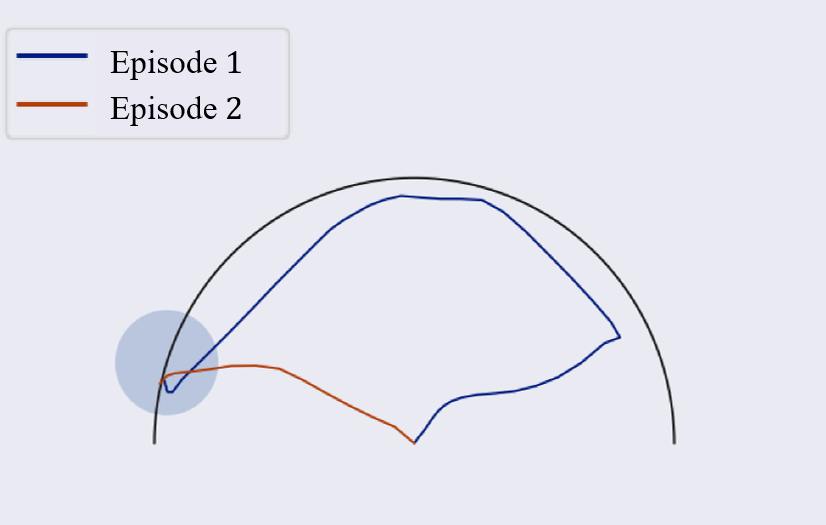}
    \caption{The HalfCircle domain (taken from \cite{offlineBRL}): the task is to navigate to a goal position that can be anywhere on the half-circle (light blue). A Bayes-optimal agent first searches along the half circle for the goal, and once found, moves directly towards it.}
    \vspace{-1em}
    \label{halfcircle}
\end{wrapfigure}

To visualize the advantage of our approach, consider the HalfCircle domain in Figure~\ref{halfcircle}, adapted from \cite{offlineBRL}: a 2-dimensional agent must navigate to a goal, located somewhere on the half-circle. A task therefore corresponds to the goal location, and the task distribution is uniform on the 1-dimensional half-circle. The bounds in \cite{RegBounds} depend on the number of states and actions, which are continuous here, and even if discretized, would result in excessive bounds. Intuitively, however, it is the low-dimensional structure of the task distribution that should determine the difficulty of the problem, and our bounds can indeed account for such. We remark that many benchmark meta RL task distributions exhibit similar low dimensional structure~\cite{RL2,MAML,offlineBRL}.

To complement our theoretical results, we demonstrate the practical potential of our approach by incorporating it in the VariBAD meta-RL method of \citet{VariBad}. In VariBAD, a variational autoencoder (VAE) is used to learn a low-dimensional latent representation of the task distribution, and a deep neural network is trained to approximate the Bayes optimal policy. We show that by estimating the task distribution using our kernel density estimation method, applied to the VAE latent space, and training the policy on sampled tasks from this distribution, we improve the generalization of the policy to tasks not seen in the training data.

\section{Background}
We specify our notation, and background on MDPs, density estimation, and dimensionality reduction.\\
\textbf{Notations:}
Estimators will be denoted with $\widehat{\times}$. $\norm{\cdot}$ denotes the Euclidean norm, other norms will be denoted explicitly. 
For a set $X$, we denote by $\probset (X)$ the set of distributions over $X$. We denote by $v_{d}$ the volume of a unit ball in $\mathbb{R}^d$. For a space $\s$ we denote $\left|\s\right|$ as the volume of the space. 

\subsection{Markov Decision Processes}
We follow the formal setting of \cite{RegBounds}. A Markov Decision Process (MDP) $M$ is a 7-tuple $M=\left(\s, \A, \C, P_{\mathrm{init}}, C, P, H \right)$, where $\s$, $\A$ and $\C$ are the state, action and cost spaces, respectively, $P_{\mathrm{init}}$ is the initial state distribution, $C: \s \times \A \rightarrow \probset(\C)$ is the cost function, $P: \s \times \A \rightarrow \probset (\s)$ is the transition function and $H\in\mathbb{Z}^{+}$ is the horizon. We denote by $P\left(c, s' \mid s, a\right)=P\left(s' \mid s, a\right)C\left(c \mid s, a\right)$ the probability of transitioning from state $s$ given action $a$ to state $s'$ with a cost of $c$, $\forall s,s'\in S, \forall a\in A$ and $\forall c\in \C$. When it is not clear from the context, we add the subscript $\times_M$ for the variables in the 7-tuple to indicate that they correspond to the particular MDP $M$. 
We denote the space of MDPs as $\M$, and use the term MDP and task interchangeably. An agent interacts with an MDP at discrete time steps: the state at time $t$ is $s_t$, where $s_0\sim P_{\mathrm{init}}$, the agent chooses action $a_t$, and then the state transitions to $s_{t+1}\sim P(\cdot | s_t,a_t)$ and cost $c_t \sim C(\cdot | s_t,a_t)$ is observed.
We denote the history at time $t$ as $h_{t}=\left\{s_{0}, a_{0}, c_{0}, s_{1}, a_{1}, c_{1} \ldots, s_{t}\right\}$, and the space of possible histories at time $t$ as $\mathcal{H}_t$.
% \\After sampling an initial state $s \sim P_{\mathrm{init}}$, 
The agent acts based on a history-dependent policy $\pi: \mathcal{H}_t \rightarrow \probset (\A)$.
Every $H$ steps, the episode ends and the state is reset, meaning it is sampled again from $P_{\mathrm{init}}$. The agent's goal in an MDP is to minimize the expected cumulative cost $\mathbb{E}_{\pi , M}\left[\sum_{t=0}^{T-1} c_t\right]$. Note that the MDP horizon $H$ and the horizon $T$ of the cumulative cost might be different -- this is important for our subsequent development, where the agent will face an unknown MDP, and can therefore improve its performance between episodes.

\subsection{Kernel Density Estimation} \label{sec:KDE}
%What is density estimation, KDE, and bounds
Density estimation is the approximation of an unknown probability density function $f(x)$, $x \in \mathbb{R}^d$, based on $n$ i.i.d.~samples from it $\{X_i\}_{i=1}^n$. One popular approach to density estimation is kernel density estimation (KDE \cite{KDE}). Let $K(x)$ denote a kernel function, which is a probability measure over $\mathbb{R}^d$. Let $\mathbf{H_0}$ denote a $d \times d$ positive definite and symmetric matrix satisfying $det\left(\mathbf{H_0}\right)=1$, referred to as a unit bandwidth matrix. Following the formulation of Jiang \cite{KDEBounds}, we define the KDE with bandwidth $h>0$ as follows:
% \begin{equation}
% \begin{aligned}
% \widehat{f}_{\mathbf{H}}(x) &:=\frac{1}{N} \cdot|\mathbf{H}|^{-d / 2} \sum_{i=1}^{N} K\left(\mathbf{H}^{-1 / 2}\left(x-X_{i}\right)\right) \\
% &=\frac{1}{N \cdot h^{d}} \sum_{i=1}^{N} K\left(\frac{\mathbf{H}_{\mathbf{0}}^{-1 / 2}\left(x-X_{i}\right)}{h}\right)
% \end{aligned}
% \end{equation}
\begin{small}
% \begin{equation*}
$
\begin{aligned}
\widehat{f}_{\mathbf{H_0},h}(x) = \frac{1}{n \cdot h^{d}} \sum_{i=1}^{n} K\left(\frac{\mathbf{H}_{\mathbf{0}}^{-1 / 2}\left(x-X_{i}\right)}{h}\right).
\end{aligned}
$
% \end{equation*}
\end{small}
% Where $K$ is the kernel function (which is a probability measure over $\mathbb{R}^d$) and $\mathbf{H_0}$ is a unit bandwidth, meaning it is a $d \times d$ positive definite, symmetric matrix with  $det\left(\mathbf{H_0}\right)=1$. 
Intuitively, $h$ controls the bandwidth size, while $\mathbf{H_0}$ allows for asymmetric bandwidth along different dimensions. We make this clearer with a simple example:

\begin{example} \label{example:GaussianKDE} [Symmetric Gaussian KDE]
For a Gaussian kernel function: $K(x)=\frac{1}{\left(\sqrt{2\pi}\right)^{d}}e^{-\frac{1}{2} x^Tx},$ bandwidth matrix $\mathbf{H}_0 = \mathbf{I}$, and bandwidth $h$, we obtain the Gaussian KDE:
% \begin{equation*}
$
    \widehat{f}_G(x) := \frac{1}{n \left(h \sqrt{2\pi}\right)^{d}} \sum_{i=1}^{n} e^{-\frac{1}{2h^2} \left(x-X_i\right)^T\left(x-X_i\right)}.
$
% \end{equation*}
This popular instance of the KDE can be understood as placing a Gaussian distribution with zero mean and standard deviation $h$ over each sample $X_i$.
\end{example}

% where the KDE is constructed by placing a Gaussian distribution with $\sigma=h$ over each sample $X_i$.

\citet{KDEBounds} provided finite sample bounds for the $L_{\infty}$ norm of the KDE error under mild assumptions.

\begin{assumption} \label{less than inf}
    $\|f\|_{\infty}<\infty$
\end{assumption}

\begin{assumption} \label{assumption:SpherSymKernel}
    The kernel function is spherically symmetric and non-increasing, meaning there exists a non-increasing function $k: \mathbb{R}_{\geq 0} \rightarrow \mathbb{R}_{\geq 0} \text { such that } K(x)=\kfunc (\norm{x}) \text { for } x \in \mathbb{R}^d$. Further, the kernel function decays exponentially, meaning there exists $\rho, C_{\rho}, t_{0}\in \mathbb{R}>0$ such that $\kfunc(t) \leq C_{\rho} \cdot \exp \left(-t^{\rho}\right), \forall t>t_0$. 
\end{assumption}

Notice that the Gaussian kernel in Example \ref{example:GaussianKDE} satisfies Assumption \ref{assumption:SpherSymKernel}.

\begin{assumption} \label{Holder}
    $f$ is $\alpha$-H\"{o}lder continuous, meaning there exists  $0<\alpha \leq 1$, $C_{\alpha}>0$, such that $\left|f(x)-f\left(x^{\prime}\right)\right| \leq$ $C_{\alpha}\norm{x-x^{\prime}}^{\alpha}$, $\forall x, x^{\prime} \in \mathbb{R}^d$.
\end{assumption}

\begin{theorem}\label{theorem:KDEorigBounds} [Theorem 2 of \citet{KDEBounds}] Under Assumptions \ref{less than inf} - \ref{Holder}, there exists a positive constant $C^{\prime} \equiv C^{\prime}\left(d, \norm{f}_{\infty}, C_{\alpha}, \alpha, K\right)$ such that the following holds with probability at least $1-1 / n$, uniformly in $h>(\log n / n)^{1 / d}$ and valid unit bandwidths matrices $\mathbf{H}_{0}$:
\begin{equation}\label{eq:KDE_bound}
\norm{\widehat{f}_{\mathbf{H_0},h}(x)-f(x)}_\infty<C^{\prime} \cdot\left(\frac{h^{\alpha}}{\sigma_{min}^{\alpha / 2}}+\sqrt{\frac{\log n}{n \cdot h^{d}}}\right),
\end{equation}
where $\sigma_{min}$ is the smallest eigenvalue of $\mathbf{H_0}$. 
\end{theorem}
The first term on the right hand side of \eqref{eq:KDE_bound} is a bias term, which can be reduced by reducing $h$. However, the second term will grow when $h$ is reduced. In general, the sample complexity under an optimal bandwidth scales exponentially in the dimension $d$ (see Lemma \ref{lemma:OptimalBandwidth} for a specific example).

\subsection{Principal Component Analysis} \label{section:PCAback}
Principal component analysis (PCA \cite{PCAOrig,PCAbounds}) is a popular linear dimensionality reduction technique. For a $d$-dimensional random variable $X$, consider $n$ i.i.d. samples of $X$, denoted $\{X_i\}_{i=1}^n$. Reducing the dimension from $d$ to $d'$ is achieved by projection -- multiplying a vector in $\mathbb{R}^d$ by a rank $d'$ orthogonal matrix $P$. 
% In PCA, we transform the data to a low dimensional representation using a linear transformation, which minimizes the projection error of given samples.
The expected projection error of $P$ is $R(P) := \mathbb{E}\left[\|X-P X\|^{2}\right]$. Similarly, the empirical projection error of $P$ is $R_N(P) := \sum_{i=1}^n\left[\|X_i-P X_i\|^{2}\right]$. We denote by $\Sigma$ and $\widehat{\Sigma}$ the covariance and empirical covariance of $X$, respectively. We further denote by $\lambda_{1} \geq \lambda_{2} \geq \cdots \lambda_{d} > 0$ the eigenvalues of $\Sigma$, and by $\widehat{\lambda}_{1} \geq \widehat{\lambda}_{2} \geq \cdots \widehat{\lambda}_{d} > 0$ the eigenvalues of $\widehat{\Sigma}$. 

The PCA projection $\widehat{P}_{d'}$ is constructed from the eigenvectors corresponding to the $d'$ largest eigenvalues of $\widehat{\Sigma}$. Its theoretical risk is:
$
\delta^{PCA}_{d'}=\mathbb{E} \left[R\left(\widehat{P}_{d'}\right)\right]-\min _{P \in \mathcal{P}_{d'}} R(P),
$
where $\mathcal{P}_{d'}$ is the space of $d\times d$ orthogonal matrices with rank $d'\leq d$, and the expectation is with respect to the $n$ samples. \citet{PCAbounds} bound the risk for sub-Gaussian random variables.

\begin{assumption} \label{assumption:subgauss} 
$X$ is sub Gaussian, meaning that its second moment is finite and there exists a constant $C_{sg}$ such that 
\begin{small}
$\sup _{k \geq 1} k^{-1 / 2} \mathbb{E}\left[|X\cdot u |^{k}\right]^{1 / k} \leq C_{sg}\mathbb{E}\left[\left(X\cdot u \right)^{2}\right]^{1 / 2}, \forall u \in \mathbb{R}^d$\end{small}
.\footnote{This definition is equivalent to the standard sub-Gaussian definition, $P(|X|>t)\leq \exp(1 - t^2/\tilde{C}_{sg})$, where $C_{sg}$ and $\tilde{C}_{sg}$ differ by a universal constant~\cite{vershynin2010introduction}.}
\end{assumption}
% Assumption \ref{assumption:subgauss} means that a random variable $X$ has a strong tail decay. 
It is worth mentioning that any bounded random variable satisfies Assumption \ref{assumption:subgauss}.
\begin{theorem}
\label{theorem:PCABound}[Proposition 2.2 of \citet{PCAbounds}] for a random variable $X$ of dimension $d$ that satisfies Assumption \ref{assumption:subgauss}, we have that
\begin{small}
$
\delta_{d'}^{PCA}\leq \min \left(\frac{8C_{sg}^2 \sqrt{d'} tr(\Sigma)}{\sqrt{n}}, \frac{64 C_{sg}^4 tr^{2}(\Sigma)}{n\left(\lambda_{d'}-\lambda_{d'+1}\right)}\right).
$
\end{small}
\end{theorem}

\vspace{-0.5em}
\section{Problem Statement}\label{section:Problem Statement}
\vspace{-0.5em}
% \subsection{MDPs parametric representation}
In this paper we make use of a parametric representation of MDPs. We suppose there is some parametric space $\Theta$ and a mapping $g:\Theta \rightarrow \M$, which maps parameters to MDPs. We next give two examples of such parametrizations, which will be referred to throughout the text to illustrate properties of our analysis.
\begin{example}\label{example:tabular} [Tabular mapping] Consider the case where  $\s,\A,T, \C$ are fixed and finite, while only $C$ and $P$ differ between different MDPs. 
The parametric space is defined by $\Theta = \Theta_C \times \Theta_P$, where $\Theta_C \subset \mathbb{R}^{|\s|\times |\A| \times |\C|}$, and $\Theta_P \subset \mathbb{R}^{|\s|\times |\A| \times |\s|}$, such that $\forall \theta_P \in \Theta_P$, $\forall \theta_C \in \Theta_C$,  
% $\forall a \in [0,|A|-1]$ and $\forall s \in [0,|S|-1]$: 
$\forall a \in \A$ and $\forall s \in \s$: 
$\theta_P(a, s, \cdot)$ is on the $|\s|$-simplex and $\theta_C(a, s, \cdot)$ is on the $|\C|$-simplex.

The parametric mapping is defined such that: $P(s'|s,a) = \theta_P(s,a,s')$ and $C(c|s,a) = \theta_C(s,a,c)$. 
\end{example}
The tabular mapping in Example \ref{example:tabular} does not assume any structure of the MDP space, and is common in the Bayesian RL literature, e.g., in the Bayes-adaptive MDP model of \citet{duff2002optimal}.
The next example considers a structured MDP space, and is inspired by the domain in Figure \ref{halfcircle}.
\begin{example} \label{example:halfcircleparametric}[Half-circle 2D navigation] 
Consider the following: 
% constant $S,A$ and $P$, where 
$\s,\A = \mathbb{R}^2$, and $P(a,s,s')=1$ for $s'=s+a$ $\forall s \in \s$, $\forall a \in \A$. The parameter space is $\Theta = \left[0, \pi\right]$. The parametric mapping $C(a,s) = 1$, $\forall a\in \A$, $\forall s \in \s$ such that 
% $\norm{s-\begin{pmatrix}
% R\cos{\theta}\\
% R\sin{\theta}
% \end{pmatrix}}_2 \leq r$, 
$\norm{s-[
R\cos{\theta}, R\sin{\theta}]}_2 \leq r$, where $r,R \in \mathbb{R}$ .  
\end{example}

\subsection{The Learning Problem}
We first define the cumulative cost for a history-dependent policy $\pi$ and an MDP $M\in \M$ as
% $$
%  L_{M, \pi}=\mathbb{E}_{\pi, M}\left[\sum_{t=0}^{T} C_{M}\left(s_{t}, a_{t}\right)\right].
% $$
$
 L_{M, \pi}=\mathbb{E}_{\pi, M}\left[\sum_{t=0}^{T-1} c_t \right].
$
We follow the Bayesian RL (BRL) formulation~\cite{ghavamzadeh2016bayesian}, and assume a prior distribution over the MDP parameter space $f\in \probset ( \Theta )$. 
Our objective is the expected cumulative cost over a randomly sampled MDP from the prior:
\begin{small}
\begin{equation*}
    \mathcal{L}_f(\pi)=\mathbb{E}_{\theta \sim f} \left[ L_{g\left(\theta \right), \pi}\right] =  \mathbb{E}_{\theta \sim f} \left[ \mathbb{E}_{\pi, M=g\left(\theta\right)}\left[\sum_{t=0}^{T-1} c_t \right]\right],
\end{equation*}
\end{small}
where we used the parametric representation introduced above.
A policy $\pi_{BO} \in \argmin_{\pi}\mathcal{L}_f(\pi)$ is termed \textit{Bayes optimal}. If the prior distribution is known, one can calculate the Bayes optimal policy~\cite{ghavamzadeh2016bayesian}. However, in our setting we assume that $f$ is not known in advance, motivating the following learning problem.
% Notice that in our setting we cannot calculate the BRL objective and $\pi_{BO}$ since we don't know the prior distribution of tasks.

We are given a training set of MDPs, $\left\{\theta\right\}_{i=1}^N$, sampled independently from the prior $f(\theta)$. Our goal is to use these MDPs to calculate a policy that minimizes the regret:
\begin{small}
$$
\mathcal{R}(\pi)=\mathcal{L}_f(\pi)-\mathcal{L}_f\left(\pi_{\mathrm{BO}}\right)= \mathbb{E}_{\theta \sim f}\left[\mathbb{E}_{\pi , M=g(\theta)}\left[\sum_{t=0}^{T-1} c_t\right]
- \mathbb{E}_{\pi_{\mathrm{BO}} , M=g(\theta)}\left[\sum_{t=0}^{T-1} c_t \right]\right]
$$
\end{small}
Note that the expectation above is with respect to the true prior, therefore, the regret can be interpreted as how well the policy calculated from the training data \textit{generalizes} to an unseen test MDP.

\section{Generalization Bounds} \label{sec:genbounds}
We next provide PAC bounds for the learning problem of 
% \subsection{Density estimation based bound}
% Given the problem setting from 
Section \ref{section:Problem Statement}. Our general idea is to first estimate the prior distribution $f(\theta)$ from the training set using KDE, and then solve for the Bayes optimal policy with respect to the KDE-estimated distribution instead of the real prior. For an estimator $\widehat{f}(\theta)$ of the real prior $f(\theta)$, we define the estimated Bayes optimal policy:  $\pi_{\widehat{f}}^* \in \argmin_{\pi}\mathcal{L}_{\widehat{f}}(\pi)$.
% BRL problem with the estimator as the prior. 
% We hypothesize that this will lead to better generalization to unseen tasks as the prior estimation will act as a regularization. We support this claim with theoretical analysis of the suggested scheme and with empirical experiments.

%\begin{lemma} \label{LossL1Bounds}
%Let $f_1(M)$ and $f_2(M)$ be two priors over the MDPs, $C_{max} = \max(C\left(s_t,a_t\right))$:
%$$
%\left|\mathcal{L}_{f_1}(\pi)-\mathcal{L}_{f_2}\left(\pi\right)\right|\leq
%   C_{max} T \norm{f_1-f_2}_1
%$$
%\end{lemma}
We start by showing that we can bound the regret of an estimated Bayes optimal policy, as a function of the estimation error of the prior itself. The proof, detailed in Section \ref{proof:L1_BO_bound} in the supplementary, is a simple application of norm inequalities, and exploiting the fact that the total cost is bounded.
\begin{lemma} \label{lemma:L1_BO_bound}
Let $\widehat{f}\in \probset (\mathbb{R}^d)$ be an estimator of the real prior $f$ over the parametric space $\Theta$. We have that:
% \begin{equation*}
% \begin{split} 
$
    \mathcal{R}(\pi_{\widehat{f}}^*) = \mathcal{L}_{f}(\pi_{\widehat{f}}^*) - \mathcal{L}_{f}(\pi_{\mathrm{BO}})\leq 2 C_{max} T \norm{f-\widehat{f}}_1,
$
% \end{split} 
% \end{equation*}
and for a bounded parametric space of volume $\left|\Theta\right|$ we have:
% \begin{equation*}
% \begin{split} 
$
    \mathcal{R}(\pi_{\widehat{f}}^*) = \mathcal{L}_{f}(\pi_{\widehat{f}}^*) - \mathcal{L}_{f}(\pi_{\mathrm{BO}})\leq 2 C_{max} T \left|\Theta \right| \norm{f-\widehat{f}}_{\infty}.
$
% \end{split} 
% \end{equation*}
\end{lemma}

While the assumption of a finite parametric space volume is reasonable in practice, the volume size, appearing in the $\norm{\cdot}_{\infty}$ bound that we shall use in the proceeding analysis, grows exponentially with the dimension of $\Theta$. In Section \ref{sec:PCA} we will discuss how, under some assumptions, we can relax this.

% \subsection{KDE based bound}
Lemma \ref{lemma:L1_BO_bound} provides a convenient framework for bounding the regret using any density estimation technique with a known bound.
We now consider the special case of the KDE (cf. Section~\ref{sec:KDE}) with a Gaussian kernel function and an optimal selection of bandwidth, as calculated in the following lemma.
We focus on the Gaussian kernel both for simplicity and because it is a popular choice in practice. Our method can be extended to any kernel function that satisfies assumption \ref{assumption:SpherSymKernel}.

\begin{lemma} \label{lemma:OptimalBandwidth}
The optimal KDE bandwidth is (up to a constant independent of $n$) $h^*=\left(\nicefrac{\log{n}}{n}\right)^{\frac{1}{2\alpha+d}}$.
\end{lemma}
The following result bounds the estimation error for this case. The proof, detailed in Section \ref{proof:gaussianKDEbounds} of the supplementary material, is based on Theorem 1 from \cite{KDEBounds}, where the constants are calculated explicitly for the Gaussian kernel case.
\begin{lemma} \label{lemma:gaussianKDEbounds} Under Assumptions \ref{less than inf} and \ref{Holder}, for a parametric space with finite volume $\left| \Theta \right|$ and a KDE with a Gaussian kernel $K(u)=\frac{e^{-\frac{1}{2} u^Tu}}{\left(2\pi\right)^{\frac{d}{2}}}$, $\mathbf{H_0}=\mathbf{I}$, and an optimal bandwidth $h^*$, we have that with probability at least $1-1/n$:
\begin{small}
% \begin{equation*}
$
    \sup _{x \in \mathbb{R}^{d}}\left|\widehat{f}_G(x)-f(x)\right| \leq C_d \cdot \left(\frac{\log{n}}{n}\right)^{\frac{\alpha}{2\alpha+d}},
$
% \end{equation*}
\end{small}
where $C_d = C_{\alpha}2^{\frac{\alpha-1}{2}} + \frac{16d \sqrt{C_{\alpha}\Delta_{max}^{\alpha}\left(\Theta\right) +\frac{1}{\left| \Theta \right|}}}{\sqrt{2}\left(2\pi\right)^{\frac{d}{4}}}+\frac{64d^2}{(2\pi)^{\frac{d}{2}}}$, and $\Delta_{max}\left(\Theta \right)$ is the maximal $L_1$ distance between any two parameters in $\Theta$.
\end{lemma}

\begin{remark} \label{remark:Cd}
Usually, $\Delta_{max}\left(\Theta \right)$ will scale polynomially (or sub-polynomially) with $d$. For example, in case the parametric space is a $d$-dimensional hypercube with edge length $B$, we have: $\Delta_{max} = \sqrt{d} B$. One parametric case that can be represented as a hyper cube is Example \ref{example:tabular}, where the edge length is $B=1$, and $d=\left|\s \right|^2 \left|\A \right| + \left|\s \right|\left|\A \right|\left|\C \right|$. In such cases, for large enough $d$, the first term in the equation for $C_d$ will dominate, and therefore $C_d \approx C_{\alpha}2^{\frac{\alpha-1}{2}}$.
\end{remark}

By combining Lemma \ref{lemma:L1_BO_bound} and \ref{lemma:gaussianKDEbounds} we obtain a regret bound -- the main result of this section.
\begin{theorem} \label{KDERegretBounds}
For a prior $f(\theta)$ over a bounded parametric space $\Theta$ that satisfies Assumptions \ref{less than inf} and \ref{Holder}, and a Gaussian KDE with optimal bandwidth, we have that with probability at least $1-1/n$:
$$
     \mathcal{R}_T(\pi_{\widehat{f}_G}^*) \leq 2 C_{max} T \left| \Theta \right| C_d \cdot \left(\frac{\log{n}}{n}\right)^{\frac{\alpha}{2\alpha+d}}.
$$
\end{theorem}

\begin{remark} \label{remark:truncatedremark}
While Theorem \ref{KDERegretBounds} assumes a bounded parametric space, the resulting KDE estimate is not necessarily bounded (e.g., when using a Gaussian kernel). In Section \ref{proof:truncatedremark} of the supplementary, we show that the result of Theorem \ref{KDERegretBounds} also holds when truncating the KDE estimate to a support $\Theta$.
\end{remark}

We next compare Theorem \ref{KDERegretBounds} with the the bounds of \citet{RegBounds}. We recall that \cite{RegBounds} considered a model-free approach that learns a history-dependent policy $\widehat{\pi}_{reg}^{*}$ on the the training domains, with $L_2$ policy regularization. 
\citet{RegBounds} only considered finite state, action, and cost spaces, a setting equivalent to our tabular representation in Example \ref{example:tabular}, where the parametric space dimension is $d=\left|\s \right|^2 \left|\A \right| + \left|\s \right|\left|\A \right|\left|\C \right|$. Corollary 1 in \cite{RegBounds} shows that with probability at least $1-1/n$, $
    \mathcal{R}_{T}\left(\widehat{\pi}^{*}\right) \leq 2 \sqrt{\aaaiparam}\cdot n^{-\frac{1}{4}} T+2 \sqrt{\aaaiparam}n^{-\frac{3}{4}}+\left(\frac{4 \sqrt{\aaaiparam}}{n^{-\frac{1}{4}}}+3 C_{\max } T\right) \left(\frac{\log{n}}{2 n}\right)^{\frac{1}{2}}
$, where $\aaaiparam=2 q^{2T} C_{\max }^{2} T^{2}|\A|$, and $q=\sup _{M, M^{\prime} \in \M, s, s^{\prime} \in \s, a \in \A, c \in \mathcal{C}} \nicefrac{P_{M}\left(s^{\prime}, c \mid s, a\right)}{P_{M^{\prime}}\left(s^{\prime}, c \mid s, a\right)}$.\footnote{The bound in \cite{RegBounds} also has a term $\lambda$ for the weight of the $L_2$ regularization term. Here, we present the result for the optimal $\lambda$, which can be calculated similarly to the optimal bandwidth in Lemma \ref{lemma:OptimalBandwidth}.}\\
At first sight, the exponent $\frac{\alpha}{2\alpha+d}$ in our bound compared to the $\frac{1}{4}$ in \cite{RegBounds} is significantly worse. The intuitive reason for this is that Theorem \ref{KDERegretBounds} builds on KDE, where it is natural to expect an exponential dependence on the dimension $d$ when estimating the density $f(\theta)$.
Yet, one must also consider the constants. Assuming that $q$ is finite places a severe constraint on the space of possible MDPs -- essentially, each cost and transition must be possible in all MDPs in the prior! \citet{RegBounds} claim that $q$ may be made finite by adding small noise to every transition, but in this case $q$ becomes $O(|\s||\C|)$, leading to an exponential $O((|\s||\C|)^{T})$ term in the bound of \cite{RegBounds}, while our bound scales linearly with $T$. The intuitive reason for this exponential dependence on $T$ is that when learning a history-dependent policy, the space of possible histories grows exponentially with $T$.

To summarize, for a small $T$, and when the structure of $\M$ is such that $q$ is finite, the model-free approach of \cite{RegBounds} seems preferable. However, when the dimension $d$ is small, and $T$ is large, our model based approach has the upper hand.

\begin{remark} \label{remark:FiniteMDP}
\citet{RegBounds} also considered a case where $\M$ is a finite set of size $\left| \M \right|$, and in this case, their
Corollary 2 shows that for $0<\alpha<1$, with probability $1-\nicefrac{1}{n^{\alpha}}$, 
\begin{small}
$
\mathcal{R}_{T}\left(\hat{\pi}^{*}\right) \leq \left(2+\sqrt{\frac{48 C_{max}^{3}|\A|}{2 P_{\min }}}\right) T^{\frac{4}{3}} n^{-\frac{1-\alpha}{3}},
$
\end{small}
where $P_{min} = \min_{M\in \M} P(M)$.\footnote{We present the result of \cite{RegBounds} for the optimal regularization coefficient, and ignore a subleading $n^{-\frac{1-a}{2}}$ term.}
In the case of a discrete and finite parametric space there is no need for the KDE, but we can still use our model based approach, and estimate the prior using the empirical distribution $\widehat{P}_{emp}(M) = \widehat{n}(M)/n$, $\forall M\in \M$, where $\widehat{n}(M)$ is the number of occurrences of $M$ in the training set. We can bound the $L_1$ error of this estimator using the Bretagnolle-Huber-Carol inequality~\cite{vaart1996weak} (the full proof is outlined in Section \ref{proof:FiniteMDP} in the supplementary), and obtain that with probability at least $1-\nicefrac{1}{n^{\alpha}}$, we have that
$
 \mathcal{R}_T(\pi_{\widehat{P}_{emp}}^*) \leq 2C_{max}T\sqrt{2\left(\alpha\log{\left(n\log{2}\right)} +|\M|+1\right)/{n}}.
$
Observe that the dependence on $T$ and $C_{max}$ in this bound is better than in the bound of \cite{RegBounds}, and we do not have the $P_{min}$ term in the denominator, which could be very small if some MDPs in the prior are rare, and is at most $\nicefrac{1}{\left| \M \right|}$. Ignoring the $\log$ term, our dependence on $n$ is also better for every $\alpha>0$.
\end{remark}

We conclude this section by pointing out that, as exemplified in Remark \ref{remark:FiniteMDP}, our approach can be generalized to density estimation techniques beyond KDE, so long as their error can be bounded.

\subsection{Bounds for Parametric Spaces with Low Dimensional Structure} \label{sec:PCA}

The sample complexity in the general bound of Theorem \ref{KDERegretBounds} grows exponentially with the dimension of the parameter space $\Theta$. 
In many practical cases however, such as the HalfCircle domain of Example \ref{example:halfcircleparametric}, there may be a low dimensional representation that encodes most of the important information in the tasks, even though the dimensionality of the parametric space is higher. In such cases, we expect that our bounds can be improved to depend on the \textit{low dimensional} representation. In the following, we approach this task by combining our model-based approach with the PCA dimensionality reduction method. PCA is a linear method, and allows for a relatively simple analysis to demonstrate our claim.
In practice, a non-linear method may be preferred. In Section \ref{section:experiments}, we verify empirically that our approach also works with non-linear deep neural network based dimensionality reduction.

We propose the following procedure.
\begin{enumerate}
\itemsep=0pt
    \item Reduce the training set dimensionality from $d$ to $d'$ using PCA
    \item Perform a Gaussian KDE estimation with optimal bandwidth in the low dimension $d'$
    \item Project back the estimated distribution to dimension $d$, to obtain the prior estimator $\widehat{f}_{G}^{d'}$
\end{enumerate} 

The main difficulty in analysing this procedure, however, is calculating the error of the estimated distribution in dimension $d$, that is, \textit{after the projection step}. We remark that projecting back is necessary, as calculating the estimated Bayes optimal policy  $\pi_{\widehat{f}}^*$ requires $g(\theta)$, where $\theta$ is in dimension $d$.
To set the stage, we first need to define the projection step explicitly. 

The PCA dimensionality reduction can be written as  $\widehat{P}_{d^{\prime}} = W_L^T\cdot W_L$, where $W_L\in \mathbb{R}^{d' \times d}$. For any $\theta\in \Theta$, we therefore have that $\theta_L = W_L \cdot \theta$ is the low dimensional representation of $\theta$, and $W_L^T \cdot \theta_L$ is the projection of $\theta_L$ back to the $d$-dimensional space.
For each $\theta_L$, we denote its inverse image as follows: $\Theta_L^{\perp}(\theta_L) = \left\{ \theta \in \Theta: W_L \cdot \theta = \theta_L\right\}$. By the law of total probability, the probability distribution of a low-dimensional $\theta_L$ is: $f_{\theta_L}(\theta_L) = \int_{\theta_L^{\perp} \in \Theta_L^{\perp}(\theta_L)} f_\theta(\theta_L^{\perp})d\theta_L^{\perp}$.

For the proceeding analysis, we require the function that maps parameters to MDPs to be smooth.
\begin{assumption} \label{assumption:StateCostAssumption}
In the MDP space $\M$, only $P$ and $C$ can differ. Furthermore, the parametric mapping is Lipshitz continuous with respect to $P_{M}\left(\cdot, \cdot \mid s, a\right)$, i.e: $\exists C_g$ such that $\forall s\in \s, a\in \A$,
$
\norm{P_{M=g(\theta_1)}\left(\cdot,\cdot \mid s, a\right)-P_{M=g(\theta_2)}\left(\cdot, \cdot \mid s, a\right)}_1 \leq C_g\norm{\theta_1-\theta_2}_1.
$
\end{assumption}

We now extend the well known \textit{Simulation Lemma} \cite{SimLemma} to the case of a history-dependent policy. This will allow us later to relate the error in the prior distribution to the error of the policy.
\begin{lemma} \label{lemma:LipCont}
For any history-dependent policy $\pi$ and any parametric mapping $g$ that satisfies Assumption \ref{assumption:StateCostAssumption}, the following holds for any $\theta_1,\theta_2 \in \Theta$:
\begin{equation}\label{eq:hist_sim_lemma}
 \left| L_{M=g(\theta_1), \pi}-L_{M=g(\theta_2), \pi} \right| \leq C_{max}C_g\norm{\theta_1-\theta_2}_1 \cdot T^2.
\end{equation}
\end{lemma}

We note that Assumption \ref{assumption:StateCostAssumption} measures closeness between the \textit{joint} distribution over costs and transitions. This is slightly different than the simulation lemma for a Markov policy~\cite{SimLemma}, where only the absolute difference between the costs is required. Intuitively, this is since a history-dependent policy can depend on observed costs, therefore even if two observed costs are very close in magnitude, the actions resulting from observing them could be very different.
% is needed for the proof of Lemma \ref{lemma:LipCont}, but, one can easily extend the setting to a broader set of MDPs and parametric mappings. 
We note that in the case where $P$ and $C$ are bounded (which is always true when $\s$, $\A$ and $\C$ are discrete), it is enough to assume Lipschitz continuity of $g$ with respect to $P$ and $C$ separately. 
% Generally, any parametric mapping that satisfies Lemma \ref{lemma:LipCont} can be used. 
We also note that there may be parameter spaces that satisfy Eq.~\ref{eq:hist_sim_lemma} without satisfying Assumption \ref{assumption:StateCostAssumption}; in such cases our proceeding results will still hold.

In the following, we would like to formally consider MDP spaces with a low-dimensional structure. This is captured by assuming that the MDPs essentially lie on a linear subspace of dimension $d'$, such that the magnitude of the variability of MDPs outside this subspace is bounded by $\epsilon$.

\begin{assumption} \label{DecayingEigen}
There exist $d'\leq d$ and  $\epsilon\in \mathbb{R}_{\geq 0}$ such that $\lambda_{d'+1}\leq \epsilon$, where $\lambda_{i}$, as defined in Section \ref{section:PCAback}, is the $i$'th largest eigenvalue of the Covariance matrix of $\theta$.
\end{assumption}

Under Assumption \ref{DecayingEigen}, and using the PCA bounds of \citet{PCAbounds} (cf. Theorem \ref{theorem:PCABound}), we can bound the dimensionality reduction error due to performing PCA on \textit{the sampled} MDPs in our data, as given by the following lemma.

\begin{lemma} \label{lemma:decayingeigen}
Under Assumptions \ref{assumption:subgauss} and \ref{DecayingEigen}, we have that:
% \begin{small}
\begin{center}
$
\mathbb{E} \left[R\left(\widehat{P}_{d'}\right)\right] \leq \min \left(\frac{8C_{sg}^2 \sqrt{d'} tr(\Sigma)}{\sqrt{n}}, \frac{64 C_{sg}^4 tr^{2}(\Sigma)}{n\left(\lambda_{d'}-\lambda_{d'+1}\right)}\right) + \epsilon \cdot (d-d').
$
\end{center}
% \end{small}
\end{lemma}
% \begin{proof}

There are two terms to the bound in Lemma \ref{lemma:decayingeigen}. The first is due to the error in performing PCA with a finite number of samples, and decays as $n$ increases. The second is due to the fact that even for a perfect PCA there is some error, as the MDPs do not lie perfectly in a low dimensional subspace. 
% achieved is bounded from below by the high dimensional residual noise $\epsilon \cdot (d-d')$.   

We next assume 
% Assumption \ref{assumption:LowDimHolder} means 
that after the PCA projection, the distribution remains H\"{o}lder continuous.
\begin{assumption} \label{assumption:LowDimHolder}
$f_{\theta_L}$ is $\alpha'$-H\"{o}lder continuous, meaning there exists  $0<\alpha' \leq 1$, $C_{\alpha'}>0$, such that $\left|f(\theta_L)-f\left(\theta_L'\right)\right| \leq$ $C_{\alpha'}\norm{\theta_L-\theta_L'}^{\alpha'}$, $\forall \theta_L, \theta'_L \in \Theta_L$.
\end{assumption}

\begin{theorem} \label{PCAKDEBounds}
Let $\widehat{f}_{G}^{d'}$ be the approximation of $f$ as defined above. Under Assumptions \ref{less than inf}, \ref{assumption:subgauss}, \ref{DecayingEigen}, and \ref{assumption:LowDimHolder} we have with probability at least $1-1/n$:

\begin{small}
% \begin{equation*}
% \begin{split}
% $
% \mathcal{R}_T\left(\pi^*_{\widehat{f}_{G}^{d'}}\right) \leq 2C_{max}T\left|\Theta_L\right| C_{d'} \left(\frac{\log{n}}{n}\right)^{\frac{\alpha'}{2\alpha'+d'}} + 2C_{max}T^2 C_g \sqrt{\min \left(\frac{8C_{sg}^2 \sqrt{d'} tr(\Sigma)}{\sqrt{n}}, \frac{64 C_{sg}^4 tr^2(\Sigma)}{n\left(\lambda_{d'}-\lambda_{d'+1}\right)}\right) + \epsilon  (d-d')},
% $
$
\mathcal{R}_T\left(\pi^*_{\widehat{f}_{G}^{d'}}\right) \leq 2C_{max}T\left|\Theta_L\right| C_{d'} \left(\frac{\log{n}}{n}\right)^{\frac{\alpha'}{2\alpha'+d'}} + 2C_{max}T^2 C_g \sqrt{\min \left(\frac{C'_{sg} \sqrt{d'}}{\sqrt{n}}, \frac{C_{sg}^{\prime 2} }{n\Delta_{\lambda, d'}}\right) + \epsilon  (d-d')},
$
% \end{split}
% \end{equation*}
\end{small}
where 
\begin{small}
$C_{d'} = C_{\alpha'}2^{\frac{\alpha'-1}{2}} + \frac{16d' \sqrt{C_{\alpha'}\Delta_{max}^{\alpha'}\left(\Theta_L\right) +\frac{1}{\left| \Theta_L \right|}}}{\sqrt{2}\left(2\pi\right)^{\frac{d'}{4}}}+\frac{64d'^2}{(2\pi)^{\frac{d'}{2}}}$, $C'_{sg}=8C_{sg}^2 tr(\Sigma)$, and $\Delta_{\lambda, d'} = \lambda_{d'}-\lambda_{d'+1}$.
\end{small}
\end{theorem}

The first term in the bound of Theorem \ref{PCAKDEBounds} is the KDE error. Note that, compared to the KDE error in Theorem \ref{KDERegretBounds}, the exponential dependence is on \textit{the low dimension} $d'$, and not on the higher dimension $d$. The second term in the bound is due to the PCA error, as discussed after Lemma \ref{lemma:decayingeigen}. This result demonstrates the potential in our approach, which accounts for structure in the parametric space.

\section{Related Work}
Meta RL has seen extensive empirical study~\cite{RL2,VariBad,MAML,rakelly2019efficient}, and the connection to Bayesian RL has been made in a series of recent works~\cite{lee2018bayesian,humplik2019meta,ortega2019meta,VariBad}. Most meta RL studies assumed infinite tasks during training (effectively drawing different random MDPs from the prior at each training iteration), with few exceptions, mostly in the offline meta RL setting~\cite{offlineBRL,li2020multitask}. 

Bayesian RL algorithms such as \cite{guez2012efficient,grover2020bayesian} sample MDPs from the true posterior, which requires knowing the true prior. 
In their comprehensive study, Simchowitz et al.~\cite{misspriors} analyse a family of posterior-sampling based algorithms, termed $n$-Monte Carlo methods, under mis-specifed priors, and bound the corresponding difference in accumulated regret. In contrast, our work considers a zero-shot setting, comparing the Bayes-optimal policies with respect to the true and the mis-specifed prior, and, more importantly, focuses on the sample complexity of \textit{obtaining} a prior that is accurate enough -- an issue that is not addressed in \cite{misspriors} for MDPs.
To the best of our knowledge, the only theoretical investigation of meta RL with finite training tasks in a similar setting to ours is the model free approach in \cite{RegBounds}, which we compare against in our work.
Our theoretical analysis builds on ideas from the study of density estimation \cite{KDE,KDEBounds} and dimensionality reduction \cite{PCAOrig,PCAbounds}. 
% \color{blue}
% These ideas were used to analyze bayesian bandits \cite{misspriors}  
% , but to our knowledge, this is the first application of them to the meta RL problem considered here. \color{black}

We note that regularization techniques inspired by the mixup method~\cite{zhang2017mixup} have been applied to meta learning~\cite{yao2021meta}, and recently also to meta RL~\cite{LDM}, with a goal of improving generalization to out-of-distribution tasks. In our experiments, we compared our approach with an approach inspired by~\cite{LDM}, and found that for in-distribution generalization, our approach worked better. That said, we believe there is much more to explore in developing effective regularization methods for meta RL. We conclude with studies on PAC-Bayes theory for meta learning~\citep{amit2018meta,rothfuss2021pacoh,farid2021pac}. These works, which have a different flavor from our PAC analysis, do not cover the meta RL problem considered here.

\section{Experiments} \label{section:experiments}
In this section we complement our theoretical results with an empirical investigation. Our goal is to show that our main idea of learning a KDE over a low dimensional space of tasks is effective also for state-of-the-art meta-RL algorithms, for which the linearity assumption of PCA clearly does not hold, and computing the optimal yet intractable $\pi_{\widehat{f}}^*$ is replaced with an approximate deep RL method.

Modern deep RL algorithms are known to be highly sensitive to many hyperparameters~\cite{henderson2018deep}, and  meta RL algorithms are not different. To demonstrate our case clearly, we chose to build on the VariBAD algorithm of \citet{VariBad}, for which we could implement our approach by replacing just a single algorithmic component, as we describe next, and thus obtain a fair comparison.

\textbf{VariBAD: }
We briefly explain the VariBAD algorithm, to set the stage for our modified algorithm to follow; we refer the reader to \cite{VariBad} for the full details. VariBAD is composed of two main components. The first is a variational autoencoder (VAE \cite{VAE}) with a recurrent neural network (RNN) encoder that, at each time step, encodes the history into a Gaussian distribution over low dimensional latent vectors. The decoder part of the VAE is trained to predict the next state and reward from the encoded history. Using the terminology in our paper, the VAE latent state can be related to the MDP parameter $\theta$, and the VAE decoder learns a model of $g(\theta)$. The second component in VariBAD is a policy, mapping the output of the VAE encoder and current state into an action. This component, which can be seen as producing an approximation of $\pi_{\widehat{f}}^*$, is trained using a policy gradient algorithm such as PPO~\cite{schulman2017proximal}.

\textbf{VariBAD Dream: }
Recall that our pipeline is to learn a KDE over the task parameters $\theta$, and then train a policy on tasks from the estimated KDE. Unfortunately, in our meta-RL setting, we do not assume that we directly know the $\theta$ representation for each task. However, the VAE in VariBAD allows for a convenient approximation. We can think of the output of the VAE encoder at the end of an episode, after a full history has been observed and the uncertainty about the task has been resolved as best as possible, as a sample of the task parameter $\theta$. Thus, we propose to build our KDE estimate over these variables. We henceforth refer to samples from the KDE as \textbf{dream environments} (as they are not present in the real data), and we note that the VAE decoder can be used to sample rewards and state transitions from these environments. Thus, we can train the VariBAD policy on both the sampled training environments, and also on dream environments. We refer to this method as \textit{VariBAD Dream}. In our implementation, we train the KDE and VariBAD components simultaneously. The full implementation details and pseudo code are in Section \ref{sec:implementation} of the supplementary.   

We emphasize that in the original VariBAD work~\cite{VariBad}, the algorithm was trained by sampling different tasks from the task distribution \textit{at each iteration}, corresponding to an infinite number of traning tasks. As we are interested in the finite task setting, our implementation of VariBAD (and VariBAD Dream) draws a single batch of $N_{train}$ tasks once, at the beginning of training, and subsequently samples tasks from within this batch for training the VAE and policy.

\textbf{Results:}
We consider the HalfCircle environment of Figure \ref{halfcircle} -- a popular task that requires learning a non-trivial Bayes-optimal policy~\cite{offlineBRL}.
% example \ref{example:halfcircleparametric}. 
In order to test generalization, we evaluated the agents on $N_{eval}=50$ sampled environments, different from the $N_{train}$ training ones. 
% All tasks' goals were sampled uniformly from the half-circle. 
% For all experiments, the agent only trains using  the training environments, and is evaluated on the evaluation environments. 
\begin{figure}
     \centering
     \vspace{-1em}
     \begin{subfigure}[l]{0.45\textwidth}
         \centering
         \includegraphics[width=\textwidth]{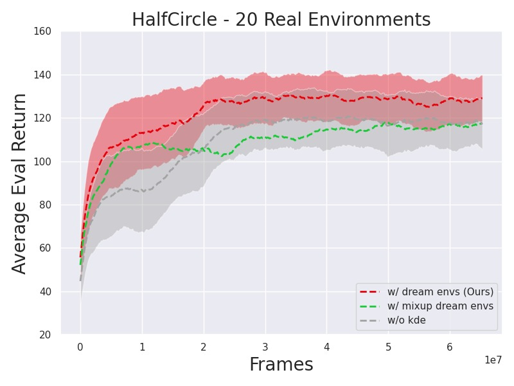}
     \end{subfigure}
     \hspace{-1em}
     \begin{subfigure}[r]{0.45\textwidth}
         \centering
         \includegraphics[width=\textwidth]{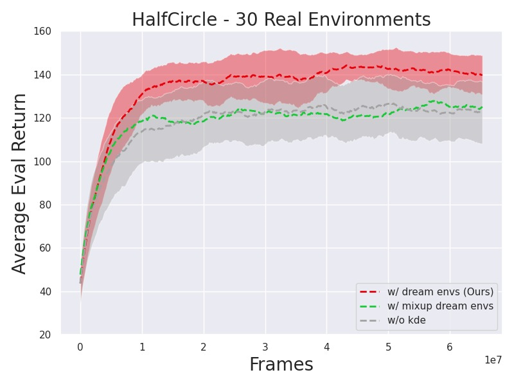}
     \end{subfigure}
     \vspace{-1em}
     \caption{Average return on HalfCircle with KDE and mixup dream environments and without dream environments. The average is shown in dashed lines, with the 95\% confidence intervals (15 random seeds). We do not show the intervals for the mixup run for visualization clarity; mixup obtained similar intervals as without dream. The full comparison is in Section \ref{sec:mixupComp} of the supplementary.}
     \label{figure:HalfCircleEval}
     \vspace{-1.5em}
\end{figure}
In Figure \ref{figure:HalfCircleEval} we show the affect of incorporating the dream environments on the average return for the evaluation environments. Evidently, our approach achieved higher return, both for $N_{train}=20$ and $N_{train}=30$. For $N_{train}=40$, the original VariBad performed near optimally, and there was no advantage for VariBad Dream to gain. For $N_{train}=10$, the sampling was too sparse, and both methods demonstrated comparable failure in generalizing to the test environments. We emphasize that while the performance advantage of VariBAD Dream is modest, it is remarkable that the KDE regularization demonstrates consistent improvement, as the VAE training and PPO optimization already include significant implicit and explicit regularization mechanisms.

% We observed that when taking more than 30 real environments in this scenario VariBad don't overfit, so there is no benefit in our regularization technique.   

% Besides VariBad and VariBad Dream 
Additionally, we evaluated a method inspired
% also compare with a method, motivated 
by mixup~\cite{zhang2017mixup,yao2021meta,LDM}, where dream environments are created by a random weighted average of latent vectors (instead of the KDE). 
% the real environments . 
Interestingly, our KDE approach outperformed this method, which may be the result of the geometry of the task distribution -- there are no goals within the half circle, in contrast with the experiments in \cite{LDM}, where the goals were distributed \textit{inside} a rectangular area. We remark that the experiments in \cite{LDM} were designed to evaluate \textit{out-of-distribution} generalization, different from the \textit{in distribution} setting considered here.
% better than the mixup approach as well. 
In Sections \ref{sec:mixupComp} and \ref{sec:ant_goal_env} of the supplementary we provide additional evaluations and show similar results on a MuJoCo \cite{todorov2012mujoco} environment with a high dimensional state space.
% , we show a comparison between the KDE and the mixup approach, both visualizing the dream environments and evaluation performances. 

% \subsection{Access to the Parametric Space}
\textbf{Synthetic Experiments with a Known Parametric Space:} 
In our theoretical analysis, we assumed knowledge of both the parameters $\theta$ of the training MDPs, and the function $g(\theta)$ that maps parameters to MDPs. We complement our empirical investigation with synthetic experiments with VariBAD Dream, where $g(\theta)$ is known.
% In this section, we empirically evaluate our method in a similar setting to our theory.

We consider again the HalfCircle environment of Figure \ref{halfcircle}. We define the parametric space as $\Theta=\mathbb{R}^2$, where 
% $M_\theta = g\left(\theta\right)$ 
$g(\theta)$ 
prescribes a corresponding MDP with a goal at location $(x,y) = \theta$. For our experiment, we sample  $\{\theta_i\}_{i=1}^N$ i.i.d.~from the training environment distribution, and perform KDE directly on these samples. Our variant of VariBad Dream samples parameters from the KDE and generates dream MDPs using the known $g(\theta)$.
% that, along the real training environments, samples parameters from the KDE and generate more MDPs using the access to $g\left(\theta \right)$. 

Figure \ref{figure:knownparams} compares VariBAD Dream with the conventional VariBad (trained only on $\{\theta_i\}_{i=1}^N$), on unseen test environments. Evidently, the use of KDE improved the generalization performances of VariBad by a large margin, especially when the number of training environments is small. Note that the improvement is starker in this synthetic setting compared to Figure \ref{figure:HalfCircleEval}, as VariBAD Dream here does not suffer from inaccuracies due to learning a model of $g(\theta)$.

\begin{figure} [h]
     \centering
     \hspace{-1em}
     \begin{subfigure}[l]{0.33\textwidth}
         \centering
         \includegraphics[width=\textwidth]{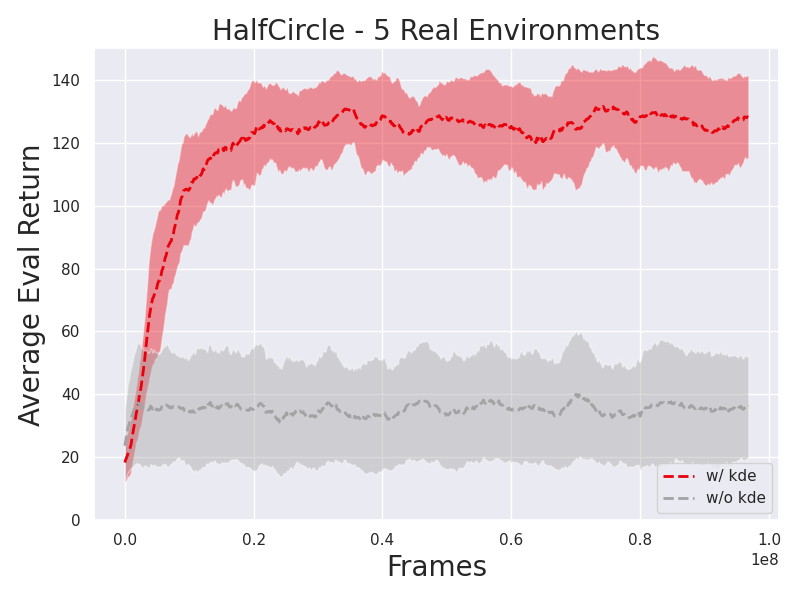}
     \end{subfigure}
     \begin{subfigure}[r]{0.33\textwidth}
         \centering
         \includegraphics[width=\textwidth]{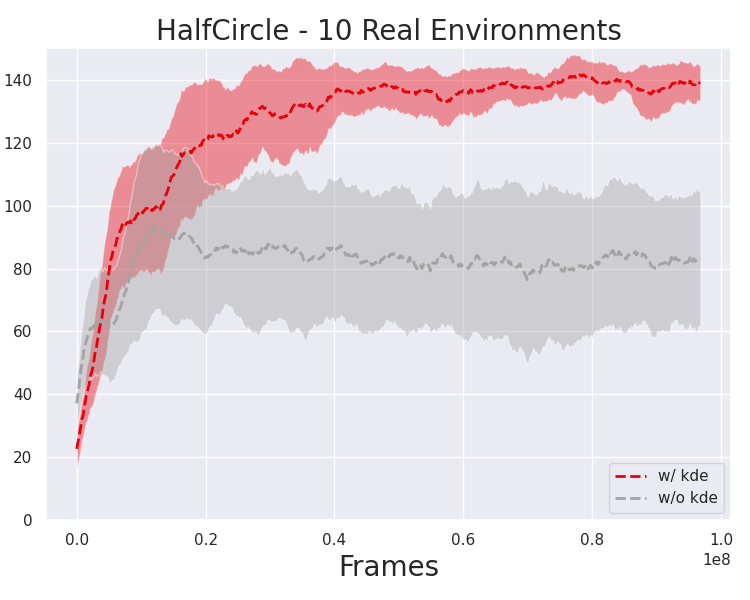}
     \end{subfigure}
     \begin{subfigure}[r]{0.33\textwidth}
         \centering
         \includegraphics[width=\textwidth]{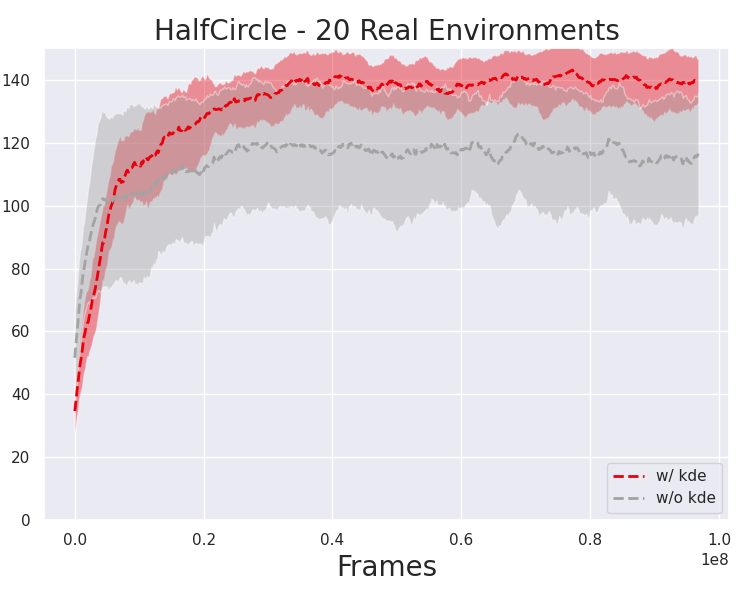}
     \end{subfigure}
     \hspace{-1em}
     \caption{Average return on half-circle of the original VariBad and the VariBAD Dream variant with access to the MDP parametric space. The average is shown in dashed lines, with the 95\% confidence intervals (using 6 random seeds).}
     \label{figure:knownparams}
     \vspace{-1em}
\end{figure}

% As the number of training tasks increases, this gain drops as the original algorithm performance increases.
% Combined with our empirical results on VariBAD Dream, we observe a clear advantage for applying KDE on the latent MDP space, as a means to improve generalization.

% \vspace{-1em}
\section{Discussion \& Future Work}
% \vspace{-1em}
We propose a model-based scheme for meta-RL with a finite sample of training tasks, where we first estimate the prior distribution of tasks, and train a Bayes optimal policy on the estimated prior. Using KDE for density estimation, we obtain state-of-the-art PAC bounds. Further, our approach can exploit low dimensional structure of the task distribution, when such exists, to obtain improved bounds. Finally, we showed that our approach can be ``plugged-in'' the VariBAD algorithm to improve generalization. 

A key takeaway from our analysis is that the \textit{dimensionality}
% number of degrees of freedom 
of the task distribution determines, with exponential dependency, 
% is the dominant factor for 
the amount of training samples required to act well.
% , rather than the size of the state and action spaces of each task. 
This insight provides a rule-of-thumb of when meta RL approaches based on task inference, such as VariBAD, are expected to work well. Indeed, recent empirical work by Mandi et al.~\citep{finetunevsmeta} claimed that in benchmarks such as RLBench~\cite{RLBench}, where tasks are very diverse, simpler meta RL methods based on fine-tuning a policy trained on diverse tasks display state-of-the-art performance. 
% which do not aim for Bayes optimality,
% and explains why state of the art meta RL algorithms often fail in environments like RLBench \cite{RLBench}, as shown by Mandi et al \cite{finetunevsmeta}. 

% We proposed a new regularization scheme for meta-RL, where we first estimate the prior distribution of tasks and use it to generate new tasks. We performed a theoretical analysis in case the density estimator is the KDE and achieved SOTA PAC bounds for the finite task meta-RL setting. Further, we showed that in cases there is a hidden structure in the tasks distribution, we can harness it using dimensionality reduction to achieve stronger bounds. Furthermore, we show empirically that our approach can be "plugged-in" the SOTA meta-RL algorithm VariBAD to improve generalization. 

There are several exciting future directions for investigation. Our theoretical analysis can be extended to more advanced density estimation and dimensionality reduction techniques, such as VAEs~\cite{rolinek2019variational}. 
% For example, going beyond linear dimensionality reduction or accounting for parameters, which don't change the cumulative return. 
Our empirical investigation hinted that regularization, such as affected by VariBAD Dream, can improve deep meta-RL algorithms. More sophisticated regularization can be developed based on prior knowledge about the possible tasks.

% There are many ways to extend our algorithm, e.g a more extensive research of how to balance real and dream environments, or incorporating prior world knowledge into the dream environments.

\section*{Acknowledgements}
This work received funding from Ford Inc., and from the European Union (ERC, Bayes-RL, Project Number 101041250). Views and opinions expressed are however those of the author(s) only and do not necessarily reflect those of the European Union or the European Research Council Executive Agency. Neither the European Union nor the granting authority can be held responsible for them. 
\newpage 
\bibliography{bib}
\bibliographystyle{plainnat}

%%%%%%%%%%%%%%%%%%%%%%%%%%%%%%%%%%%%%%%%%%%%%%%%%%%%%%%%%%%%
\section*{Checklist}

%% BEGIN INSTRUCTIONS %%%
The checklist follows the references.  Please
read the checklist guidelines carefully for information on how to answer these
questions.  For each question, change the default \answerTODO{} to \answerYes{},
\answerNo{}, or \answerNA{}.  You are strongly encouraged to include a {\bf
justification to your answer}, either by referencing the appropriate section of
your paper or providing a brief inline description.  For example:
\begin{itemize}
  \item Did you include the license to the code and datasets? \answerYes{See Section~}
  \item Did you include the license to the code and datasets? \answerNo{The code and the data are proprietary.}
  \item Did you include the license to the code and datasets? \answerNA{}
\end{itemize}

\begin{enumerate}

\item For all authors...
\begin{enumerate}
  \item Do the main claims made in the abstract and introduction accurately reflect the paper's contributions and scope?
    \answerYes{}
  \item Did you describe the limitations of your work?
    \answerYes{}
  \item Did you discuss any potential negative societal impacts of your work?
    \answerNA{}
  \item Have you read the ethics review guidelines and ensured that your paper conforms to them?
    \answerYes{}
\end{enumerate}

\item If you are including theoretical results...
\begin{enumerate}
  \item Did you state the full set of assumptions of all theoretical results?
    \answerYes{}
	\item Did you include complete proofs of all theoretical results?
    \answerYes{In the supplementary materials}
\end{enumerate}

\item If you ran experiments...
\begin{enumerate}
  \item Did you include the code, data, and instructions needed to reproduce the main experimental results (either in the supplemental material or as a URL)?
    \answerYes{In the supplementary materials}
  \item Did you specify all the training details (e.g., data splits, hyperparameters, how they were chosen)?
    \answerYes{In Section \ref{sec:implementation} of the supplementary}
	\item Did you report error bars (e.g., with respect to the random seed after running experiments multiple times)?
    \answerYes{In Section \ref{section:experiments} and Section \ref{sec:mixupComp} in the supplementary}
	\item Did you include the total amount of compute and the type of resources used (e.g., type of GPUs, internal cluster, or cloud provider)?
    \answerYes{See Section \ref{sec:compute} of the supplementary}
\end{enumerate}

\item If you are using existing assets (e.g., code, data, models) or curating/releasing new assets...
\begin{enumerate}
  \item If your work uses existing assets, did you cite the creators?
    \answerYes{See Section \ref{sec:implementation} of the supplementary}
  \item Did you mention the license of the assets?
    \answerYes{See Section \ref{sec:implementation} of the supplementary}
  \item Did you include any new assets either in the supplemental material or as a URL?
    \answerYes{See Section \ref{sec:implementation} of the supplementary}
  \item Did you discuss whether and how consent was obtained from people whose data you're using/curating?
    \answerNA{}
  \item Did you discuss whether the data you are using/curating contains personally identifiable information or offensive content?
    \answerNA{}
\end{enumerate}

\item If you used crowdsourcing or conducted research with human subjects...
\begin{enumerate}
  \item Did you include the full text of instructions given to participants and screenshots, if applicable?
    \answerNA{}
  \item Did you describe any potential participant risks, with links to Institutional Review Board (IRB) approvals, if applicable?
    \answerNA{}
  \item Did you include the estimated hourly wage paid to participants and the total amount spent on participant compensation?
    \answerNA{}
\end{enumerate}

\end{enumerate}

%%%%%%%%%%%%%%%%%%%%%%%%%%%%%%%%%%%%%%%%%%%%%%%%%%%%%%%%%%%%

\appendix
\newpage
\section{Appendix}
\subsection{Limitations}\label{sec:limitation}
In this section we summarize several limitations of our study. 

Our analysis applies to linear dimensionality reduction (PCA), while in practice, it may be that the prior lives on a low dimensional non-linear manifold. Extending our theory for such cases would require a significantly more elaborate approach. However, our experiments show that empirically, our insights still hold when the PCA is replaced with a non-linear VAE.

As discussed in Section \ref{sec:genbounds}, the bounds achieved by our method are exponential with the underlying dimension of the task distribution. The lower-bound example in Proposition 2 of \cite{RegBounds} can be adapted to our setting (by using the Tabular Mapping in Example \ref{example:tabular}), to show a problem setting where an exponential dependence on the dimension cannot be avoided, regardless of the algorithm. Thus, without additional structure in the problem, this limitation is general to meta-RL and not specific to our method.

The proposed algorithm, VariBad Dream, builds on  top of VariBad, and requires the latent space to be learned by the VariBad algorithm. We are therefore limited to environments in which VariBad performs adequately. Furthermore, in order to create the prior estimate using KDE, we used VariBAD latents gathered from the VAE posterior at the end of a VariBad rollout. This may not work well in some cases, e.g., when the task uncertainty is not resolved at the end of the episode.  

\subsection{Regret Bounds Using Prior Estimation} \label{proof:L1_BO_bound}
\textbf{Lemma \ref{lemma:L1_BO_bound}}. Let $\widehat{f}\in \probset (\mathbb{R}^d)$ be an estimator of the real prior $f$ over the parametric space $\Theta$. We have that:
$
    \mathcal{R}(\pi_{\widehat{f}}^*) = \mathcal{L}_{f}(\pi_{\widehat{f}}^*) - \mathcal{L}_{f}(\pi_{\mathrm{BO}})\leq 2 C_{max} T \norm{f-\widehat{f}}_1,
$
and for a bounded parametric space of volume $\left|\Theta\right|$ we have:
$
    \mathcal{R}(\pi_{\widehat{f}}^*) = \mathcal{L}_{f}(\pi_{\widehat{f}}^*) - \mathcal{L}_{f}(\pi_{\mathrm{BO}})\leq 2 C_{max} T \left|\Theta \right| \norm{f-\widehat{f}}_{\infty}.
$

\begin{proof} [Proof of Lemma \ref{lemma:L1_BO_bound}]
\begin{equation*}
\begin{split}
    \mathcal{L}_{f}(\pi)-\mathcal{L}_{\widehat{f}}(\pi)&=
     \mathbb{E}_{\theta \sim f(\theta)} \mathbb{E}_{\pi, M=g(\theta)}\left[\sum_{t=0}^{T-1} c_t\right] - \mathbb{E}_{\theta \sim \widehat{f}(\theta)} \mathbb{E}_{\pi, M=g(\theta)}\left[\sum_{t=0}^{T-1} c_t \right] \\&=
    \int \mathbb{E}_{\pi, M=g(\theta)}\left[\sum_{t=0}^{T-1} c_t\right] f(\theta) d\theta - \int \mathbb{E}_{\pi, M=g(\theta)}\left[\sum_{t=0}^{T-1} c_t\right] \widehat{f}(\theta) d\theta \\&=
    \int \mathbb{E}_{\pi, M=g(\theta)}\left[\sum_{t=0}^{T-1}c_t\right] \left(f(\theta)-\widehat{f}(\theta)\right) d\theta
\end{split}
\end{equation*}

Taking the absolute value:
\begin{equation*}
\begin{split}
    \left|\mathcal{L}_{f}(\pi)-\mathcal{L}_{\widehat{f}}(\pi)\right|&=
    \left|\int \mathbb{E}_{\pi, M=g(\theta)}\left[\sum_{t=0}^{T-1} c_t\right] \left(f(\theta)-\widehat{f}(\theta)\right) d\theta\right| \\&\leq
   \int  \left|\mathbb{E}_{\pi, M=g(\theta)}\left[\sum_{t=0}^{T-1} c_t\right] \right|\left|\left(f(\theta)-\widehat{f}(\theta)\right)\right| d\theta \\&\leq C_{max} T\int \left|\left(f(\theta)-\widehat{f}(\theta)\right)\right| d\theta\leq C_{max} T \norm{f_1-f_2}_1
\end{split} 
\end{equation*}

Using the above with $A=C_{max} T \norm{f-\widehat{f}}_1$:
\begin{equation*}
\begin{split}
    \mathcal{L}_{f}(\pi_{\widehat{f}}^*)-A \leq \mathcal{L}_{\widehat{f}}(\pi_{\widehat{f}}^*) \leq& \mathcal{L}_{\widehat{f}}(\pi_{f}^*) \leq  \mathcal{L}_{f}(\pi_{f}^*)+A\\
    \Rightarrow \mathcal{L}_{f}(\pi_{\widehat{f}}^*)-A\leq& \mathcal{L}_{f}(\pi_{f}^*)+A
\end{split}
\end{equation*}
Rearranging gives the first bound. 
For a finite parametric space of size $\left|\Theta\right|$, we know that $\norm{f-\widehat{f}}_1 \leq \left|\Theta\right|\cdot \norm{f-\widehat{f}}_{\infty}$, which yields the second bound.  
\end{proof}

\subsection{Optimal KDE Bandwidth} \label{proof:OptimalBandwidth}

\textbf{Lemma \ref{lemma:OptimalBandwidth}}. The optimal KDE bandwidth is (up to a constant independent of $n$) $h^*=\left(\nicefrac{\log{n}}{n}\right)^{\frac{1}{2\alpha+d}}$

\begin{proof} [Proof of Lemma \ref{lemma:OptimalBandwidth}]

The minimum value of the function $f(x)=Ax^a+Bx^b$ with $A,a,B,b \neq 0$ is achieved with:
$$
x^* = \left(-\frac{bB}{aA}\right)^{\frac{1}{a-b}} 
$$

We can use this result to achieve the optimal bandwidth for the bound in lemma \ref{theorem:KDEorigBounds}, with: $A=C^{\prime}\sigma_{min}^{-\alpha / 2}$, $a=\alpha$, $B=C^{\prime}\sqrt{\frac{\log{n}}{n}}$ and $b=-\frac{d}{2}$, resulting with:

$$
\argmin_{h\in \mathbb{R}^+}{\norm{\widehat{f}_{\mathbf{H}}(x)-f(x)}_\infty} =  \left(\frac{d^2 \log{n}}{4 \alpha^2 n}\right)^{\frac{1}{2\alpha+d}}
$$

\end{proof}

\subsection{Gaussian KDE Bounds} \label{proof:gaussianKDEbounds}

\textbf{Lemma \ref{lemma:gaussianKDEbounds}}. Under Assumptions \ref{less than inf} and \ref{Holder}, for a parametric space with finite volume $\left| \Theta \right|$ and a KDE with a Gaussian kernel $K(u)=\frac{e^{-\frac{1}{2} u^Tu}}{\left(2\pi\right)^{\frac{d}{2}}}$, $\mathbf{H_0}=\mathbf{I}$, and an optimal bandwidth $h^*$, we have that with probability at least $1-1/n$:
\begin{small}
% \begin{equation*}
$
    \sup _{x \in \mathbb{R}^{d}}\left|\widehat{f}_G(x)-f(x)\right| \leq C_d \cdot \left(\frac{\log{n}}{n}\right)^{\frac{\alpha}{2\alpha+d}},
$
% \end{equation*}
\end{small}
where $C_d = C_{\alpha}2^{\frac{\alpha-1}{2}} + \frac{16d \sqrt{C_{\alpha}\Delta_{max}^{\alpha}\left(\Theta\right) +\frac{1}{\left| \Theta \right|}}}{\sqrt{2}\left(2\pi\right)^{\frac{d}{4}}}+\frac{64d^2}{(2\pi)^{\frac{d}{2}}}$, and $\Delta_{max}\left(\Theta \right)$ is the maximal $L_1$ distance between any two parameters in $\Theta$.

\begin{proof}[Proof of Lemma \ref{lemma:gaussianKDEbounds}]

% Jiang et al \cite{KDEBounds} showed a bound for the sup-norm consistency of the KDE, but neglected the exact calculation of the constants of the bound. Here we will show the exact calculation.

% Using Jensen inequality (For $d\geq 2$): 
% $$f(x^*) \leq 2 \left(\frac{1}{2} \left(\left(-\frac{b}{a}\right)^a+\left(-\frac{b}{a}\right)^b\right)A^{-b}\cdot B^a\right)^\frac{1}{a-b}$$

We follow the proof of Theorem 2 in \cite{KDEBounds}.

We define:
$$
\check{u}_{x}(r):=f(x)-\inf _{x^{\prime} \in B(x, r)} f\left(x^{\prime}\right)
$$
and
$$
\widehat{u}_{x}(r):=\sup _{x^{\prime} \in B(x, r)} f\left(x^{\prime}\right)-f(x)
$$
the following holds \cite{KDEBounds}:
$$
\int_{\mathbb{R}^{d}} K(u) \check{u}_{x}\left(\frac{h\norm{u}}{ \sqrt{\sigma_{min}}}\right) d u \leq \frac{v_{d} \cdot C_{\alpha} h^{\alpha}}{\sigma_{min}^{\alpha / 2}} \int_{0}^{\infty} k(t) t^{d+\alpha} d t
$$
and the above equation is also valid when replacing $\check{u}_{x}$ with $\widehat{u}_{x}$

We can use Theorem 1 from \cite{KDEBounds} and get:
$$
\sup _{x \in \mathbb{R}^{d}}\left|\widehat{f_{\mathbf{H}}}(x)-f(x)\right|<\epsilon h^{\alpha}+C_{\infty} \sqrt{\frac{\log n}{n \cdot h^{d}}}
$$
Where $\epsilon=\frac{v_{d} \cdot C_{\alpha}}{\sigma_{min}^{\alpha / 2}} \int_{0}^{\infty} k(t) t^{d+\alpha} d t$, and $C_{\infty} = 8 d \sqrt{v_{d} \cdot\|f\|_{\infty}}\left(\int_{0}^{\infty} k(t) \cdot t^{d / 2} d t+1\right)+64 d^{2} \cdot k(0)$.

$\kfunc$ is the function introduced in assumption \ref{assumption:SpherSymKernel} and $v_d = \frac{\pi^{d / 2}}{\Gamma(1+d / 2)}$ is the volume of the d-dimensional unit ball, where $\Gamma$ is the Gamma function. 

% Now we can use lemma \ref{minresult} and get the exact result. In our case:

% \begin{equation} \label{exactConsts}
%     \sup _{x \in \mathbb{R}^{d}}\left|\widehat{f_{\mathbf{H}}}(x)-f(x)\right|<\left(\left(\left(\frac{d}{2\alpha}\right)^{\alpha}+\left(\frac{d}{2 \alpha}\right)^{-\frac{d}{2}}\right) A^{\frac{d}{2}}\cdot \left(C\sqrt{\frac{\log{n}}{n}}\right)^{\alpha}\right)^{\frac{2}{2\alpha+d}}
% \end{equation}

% where $A=\frac{v_{d} \cdot C_{\alpha} }{\sigma_{d}\left(\mathbf{H}_{0}\right)^{\alpha / 2}} \int_{0}^{\infty} k(t) t^{d+\alpha} d t$
For KDE with the optimal bandwidth $h^*= \left(\frac{\log{n}}{n}\right)^{\frac{1}{2\alpha+d}}$, defined in lemma \ref{lemma:OptimalBandwidth} we get:
$$
\sup _{x \in \mathbb{R}^{d}}\left|\widehat{f_{\mathbf{H}}}(x)-f(x)\right|<\epsilon h^{\alpha}+C_{\infty} \sqrt{\frac{\log n}{n \cdot h^{d}}} = \left(\epsilon + C_{\infty} \right)\cdot \left(\frac{\log{n}}{n}\right)^{\frac{\alpha}{2\alpha+d}}
$$

In the case of the Gaussian kernel presented in example \ref{example:GaussianKDE}: $K(x)=\frac{e^{-\frac{1}{2} x^Tx}}{\left(2\pi\right)^{\frac{d}{2}}}$, $k(t) = \frac{e^{-\frac{1}{2}t^2}}{\sqrt{\left(2\pi\right)^d}}$ and $\sigma_{min}=1$.

The well known formula for the moments of the Gaussian distribution: 
$$
\int_{0}^{\infty} k(t) \cdot t^{a} d t = \frac{\Gamma\left(\frac{a+1}{2}\right)}{2^{\frac{d-a+1}{2}}\pi^{\frac{d}{2}}}
$$

Using the fact that $\Gamma(x)$ is monotonically increasing $\forall x>1$:

\begin{align*}
\epsilon  &=
\frac{v_{d} \cdot C_{\alpha} }{\sigma_{min}^{\alpha / 2}} \int_{0}^{\infty} k(t) t^{d+\alpha} d t=
\frac{\pi^{d / 2}}{\Gamma(1+d / 2)} \cdot C_{\alpha} \cdot \frac{\Gamma\left(\frac{d+\alpha+1}{2}\right)}{2^{\frac{1-\alpha}{2}}\pi^{\frac{d}{2}}} \\&= 
\frac{C_{\alpha}}{\Gamma(1+d / 2)\cdot 2^{\frac{1-\alpha}{2}}} \cdot \Gamma\left(\frac{d+\alpha+1}{2}\right) \\&\leq
\frac{C_{\alpha}}{\Gamma(1+d / 2)\cdot 2^{\frac{1-\alpha}{2}}} \cdot \Gamma\left(\frac{d+2}{2}\right) = C_{\alpha}2^{\frac{\alpha-1}{2}}
\end{align*}

Notice that in our case, since the function is $\alpha$-H\"{o}lder continuous and its support size is $\left| \Theta \right|$:
$$
f_{max}-f_{min}\leq C_{\alpha} \Delta_{max}^{\alpha}\left(\Theta\right) \Rightarrow \|f\|_{\infty}\leq C_{\alpha}\Delta_{max}^{\alpha}\left(\Theta\right)+\frac{1}{\left| \Theta \right|}
$$

Where $\Delta_{max}\left(\Theta\right)$ is the maximum $L_1$ distance between two parameters in $\Theta$.

Using the fact that $\sqrt{\Gamma(2x)}>\Gamma(x)$:
\begin{align*}
C_{\infty} &= 8 d \sqrt{v_{d} \cdot\|f\|_{\infty}}\left(\int_{0}^{\infty} k(t) \cdot t^{d / 2} d t+1\right)+64 d^{2} \cdot k(0) \\&= 
8 d \sqrt{\frac{\pi^{d / 2}}{\Gamma(1+d / 2)} \cdot\|f\|_{\infty}}\left(\frac{\Gamma\left(\frac{d+2}{4}\right)}{2^{\frac{d+2}{4}}\pi^{\frac{d}{2}}}+1\right)+\frac{64d^2}{\sqrt{(2\pi)^d}}\\ &\leq
16 d \sqrt{\pi^{d / 2} \cdot\|f\|_{\infty}}2^{\frac{-d-2}{4}}\pi^{\frac{-d}{2}}+\frac{64d^2}{\sqrt{(2\pi)^d}} \leq
\frac{16d \sqrt{C_{\alpha}\Delta_{max}^{\alpha}\left(\Theta\right) +\frac{1}{\left| \Theta \right|}}}{\sqrt{2}\left(2\pi\right)^{\frac{d}{4}}}+\frac{64d^2}{(2\pi)^{\frac{d}{2}}}
\end{align*}

% And we know that for large enough d:
% $$
% \frac{d}{2\alpha}>\left(\frac{d}{2\alpha}\right)^{\alpha}>\left(\frac{d}{2 \alpha}\right)^{-\frac{d}{2}}
% $$

Concluding:
% \begin{align} \label{exactConstsGauss}
%     \sup _{x \in \mathbb{R}^{d}}\left|\widehat{f_{\mathbf{H}}}(x)-f(x)\right|&<\left(\frac{d}{\alpha} C_{\alpha}^{\frac{d}{2}}\cdot \left(\frac{32d \sqrt{C_{\alpha}B^{ d}+\frac{1}{B^d}}}{\sqrt{2}\left(2\pi\right)^{\frac{d}{4}}}\sqrt{\frac{\log{n}}{n}}\right)^{\alpha}\right)^{\frac{2}{2\alpha+d}}\\&\leq 
%     \left(\frac{d}{\alpha} C_{\alpha}^{\frac{d}{2}}\cdot \frac{32d \sqrt{C_{\alpha}B^{ d}+\frac{1}{B^d}}}{\sqrt{2}\left(2\pi\right)^{\frac{d}{4}}}\sqrt{\frac{\log{n}}{n}}\right)^{\frac{2}{d}} \\&=
%     \frac{C_{\alpha}}{\sqrt{2\pi}}\left(\frac{32d^2 \sqrt{C_{\alpha}B^{ d}+\frac{1}{B^d}}}{\alpha \sqrt{2}}\sqrt{\frac{\log{n}}{n}}\right)^{\frac{2}{d}}
% \end{align}

\begin{align*} 
    \sup _{x \in \mathbb{R}^{d}}\left|\widehat{f_{\mathbf{H}}}(x)-f(x)\right|&<\left(\epsilon + C_{\infty} \right)\cdot \left(\frac{\log{n}}{n}\right)^{\frac{\alpha}{2\alpha+d}}\\&\leq\left(C_{\alpha}2^{\frac{\alpha-1}{2}} + \frac{16d \sqrt{C_{\alpha}\Delta_{max}^{\alpha}\left(\Theta\right) +\frac{1}{\left| \Theta \right|}}}{\sqrt{2}\left(2\pi\right)^{\frac{d}{4}}}+\frac{64d^2}{(2\pi)^{\frac{d}{2}}} \right)\cdot \left(\frac{\log{n}}{n}\right)^{\frac{\alpha}{2\alpha+d}}
\end{align*}
\end{proof}
% Notice that due to the power of $\frac{2}{d}$, the first term is monotonically decreasing with d so:
% \begin{align*} 
%     \sup _{x \in \mathbb{R}^{d}}\left|\widehat{f_{\mathbf{H}}}(x)-f(x)\right|&\leq C\cdot \frac{d\sqrt{\left| \Theta \right|}}{\left(2\pi\right)^{d/4}}\left(\frac{\log{n}}{n}\right)^\frac{\alpha}{2\alpha+d}
% \end{align*}

\subsection{Bounds for a Truncated Estimator}\label{proof:truncatedremark}

\textbf{Remark \ref{remark:truncatedremark}}. The result of Theorem \ref{KDERegretBounds} also holds when truncating the KDE estimate to a support $\Theta$.

\begin{proof} [Proof of Remark \ref{remark:truncatedremark}]

Let $f_{1}\in \probset(\Theta)$ be a PDF where $\Theta$ is of finite size $\left|\Theta \right|$ and dimension $d$ and $f_{2} \in \probset(\mathbb{R}^d)$ another PDF such that $\norm{f_1-f_2}_\infty \leq U$ (where $U<1$). So:
$$
\norm{f_1-f^T_2}_\infty \leq \frac{\left(\left|\Theta \right|+1\right)\cdot U}{1- \left|\Theta \right| U}
$$
Where $f^T_2$ is the truncated version of $f_2$ ($f^T_2(\theta)=\frac{f_2(\theta)}{\int_{\theta \in \Theta} f_2(\theta)d\theta}$ for $\theta \in \Theta$, else 0)
Let $r = \int_{\theta \in \Theta} f_2(\theta)d\theta$, we can bound $1-r$:
$$
1-r = \int_{\theta\in \Theta} \left( f_1(\theta) - f_2(\theta) \right) d\theta \leq
\int_{\theta\in \Theta} \left|f_1(z) - f_2(z)\right| dz \leq \left|\Theta \right| \cdot \norm{f_1-f_2}_\infty
$$

So:
\begin{align*}
    \norm{f_1-f^T_2}_\infty &\leq \norm{f_1-\frac{f_2}{r}}_\infty=\frac{1}{r}\cdot \norm{f_2-r\cdot f_1}_\infty= \frac{1}{r}\cdot \norm{f_2-f1+\left(1-r\right)\cdot f_1}_\infty 
    \\&\leq \frac{1}{r}\left(\norm{f_2-f_1}_\infty+\norm{\left(1-r\right)\cdot f_1}_\infty \right)
\end{align*}

Concluding:
$$
    \norm{f_1-f^T_2}_\infty \leq \frac{1}{1-\left|\Theta \right|  U}\left(U+\left|\Theta \right| \cdot U \right) = \frac{\left(1+\left|\Theta \right| \right)\cdot U}{1-\left|\Theta \right| U}
$$
\end{proof}

\subsection{Bounds for Discrete and Finite Parametric Space}

\textbf{Remark \ref{remark:FiniteMDP}}. In the case of a discrete and finite parametric space there is no need for the KDE, but we can still use our model based approach, and estimate the prior using the empirical distribution $\widehat{P}_{emp}(M) = \widehat{n}(M)/n$, $\forall M\in \M$, where $\widehat{n}(M)$ is the number of occurrences of $M$ in the training set. We can bound the $L_1$ error of this estimator using the Bretagnolle-Huber-Carol inequality~\cite{vaart1996weak}, and achieve that with probability at least $1-\nicefrac{1}{n^{\alpha}}$ we have that,
$
 \mathcal{R}_T(\pi_{\widehat{P}_{emp}}^*) \leq 2C_{max}T\sqrt{2\left(\alpha\log{\left(n\log{2}\right)} +|\M|+1\right)/{n}}.
$

\begin{proof} [Proof of Remark \ref{remark:FiniteMDP}] \label{proof:FiniteMDP}
$$
\operatorname{Pr}\left[\sum_{i=1}^{|\M|}\left|\frac{\hat{n}_{i}}{n}-p_{i}\right| \geq \lambda\right] \leq 2^{|\M|+1} e^{-n \lambda^{2} / 2}
$$

\begin{align*}
    2^{|\M|+1} e^{-n \lambda^{2} / 2} = n^{-\alpha}
\end{align*}

$$
\frac{1}{\log{(2)}} e^{|\M|+1-\frac{n \lambda^{2}}{2}} = n^{-\alpha}
$$

$$
 |\M|+1-n \lambda^{2} / 2 = \log{\left(\log{(2)} n^{-\alpha}\right)}
$$

$$
 n \lambda^{2} / 2 = \alpha \log{\left(\log{(2)} n\right)}+|\M|+1
$$

$$
 \lambda = \sqrt{\left(\alpha \log{\left(n\log{2}\right)} +|\M|+1\right)\frac{2}{n}}
$$

So with probability at least 1-1/n we have:
$$
\sum_{i=1}^{|\M|}\left|\frac{\hat{n}_{i}}{n}-p_{i}\right| \leq \sqrt{\left(\alpha \log{\left(n\log{2}\right)} +|\M|+1\right)\frac{2}{n}}
$$
By using Lemma \ref{lemma:L1_BO_bound} we get the result.
\end{proof}

\subsection{The History Dependent Simulation Lemma}\label{proof:LipCon}

\textbf{Lemma \ref{lemma:LipCont}}. For any history-dependent policy $\pi$ and any parametric mapping $g$ that satisfies Assumption \ref{assumption:StateCostAssumption}, the following holds for any $\theta_1,\theta_2 \in \Theta$:
\begin{equation*}
 \left| L_{M=g(\theta_1), \pi}-L_{M=g(\theta_2), \pi} \right| \leq C_{max}C_g\norm{\theta_1-\theta_2}_1 \cdot T^2.
\end{equation*}

\begin{proof} [Proof of Lemma \ref{lemma:LipCont}] 
For ease of notation, our proof is for the case of discrete state and action spaces and discrete range of the cost function. Yet, this proof can easily be extended to the more general continuous case by replacing the sums with integrals. 

The history at step t:
$$
h_t = \left\{s_{0}, a_{0}, c_{0}, s_{1}, a_{1}, c_{1} \ldots, s_{t}\right\}
$$

The cost distribution at step t for a deterministic, history-dependent policy and mdp $M$: 
$$
C_M^{\pi}\left(h_{t}\right) := C_M\left(c_{t} \mid s_{t}, \pi(h_{t})\right)
$$
And the average cost:
$$
\bar{C}_M^{\pi}\left(h_{t}\right) := \sum_{c_t} c_t\cdot C_M\left(c_{t} \mid s_{t}, \pi(h_{t})\right)
$$

The value function:
$$
V_{t, M}^{\pi}\left(h_{t}\right)=\mathbb{E}_{\pi , M}\left[\sum_{t^{\prime}=t}^{T}C_M^{\pi}\left(h_{t'}\right) \mid h_{t}\right]
$$

The RL loss:
$$
 L_{M, \pi} = \mu^T V_{0, M}^{\pi} =   \mathbb{E}_{\pi, M}\left[\sum_{t=0}^{T-1} C_M^{\pi}\left(h_{t}\right)\right]
$$

Where $\mu$ is the initial state distribution.

For $t=T-1$:
$$
V_{T-1, M}^{\pi}\left(h_{T-1}\right) = \bar{C}_M^{\pi}\left(h_{T-1}\right)
$$

For $t<T-1$:
$$
V_{t, M}^{\pi}\left(h_{t}\right)=\bar{C}_M^{\pi}\left(h_{t}\right)+\sum_{c_t, s_{t+1}} P_M\left(c_t, s_{t+1} \mid s_{t}, \pi(h_{t})\right) V_{t+1, M}^{\pi}\left(\left\{h_{t}, \pi(h_{t}), c_t, s_{t+1}\right\}\right)
$$

For two MDPs $M$ and $M'$ such that $\forall s \in \s$ and $\forall a \in \A$: 
$$\sum_{c, s'} \left|P_M\left(c, s' \mid s, a\right) - P_{M'}\left(c, s' \mid s, a\right) \right|=\epsilon$$

We have:

\begin{align*}
\bar{C}_M^{\pi}\left(h_{t}\right)-\bar{C}_{M'}^{\pi}\left(h_{t}\right) =& \sum_{c_t} c_t\cdot \left(P_M\left(c_{t} \mid s_{t}, \pi(h_{t})\right)-P_{M'}\left(c_{t} \mid s_{t}, \pi(h_{t})\right)\right)\\
\left|\bar{C}_M^{\pi}\left(h_{t}\right)-\bar{C}_{M'}^{\pi}\left(h_{t}\right)\right| \leq& C_{max}\cdot \sum_{c_t} \left|P_M\left(c_{t} \mid s_{t}, \pi(h_{t})\right)-P_{M'}\left(c_{t} \mid s_{t}, \pi(h_{t})\right)\right| \leq C_{max}\epsilon
\end{align*}

For ease of notation we define $h_{t+1} = \left\{h_{t}, \pi(h_{t}), c_t, s_{t+1}\right\}$:
\begin{align*}
V_{t, M}^{\pi}\left(h_{t}\right)-V_{t, M'}^{\pi}\left(h_{t}\right)&=\bar{C}_M^{\pi}\left(h_{t}\right)+\sum_{c_t, s_{t+1}} P_M\left( c_t, s_{t+1} \mid s_{t}, \pi(h_{t})\right) V_{t+1, M}^{\pi}\left(h_{t+1}\right)\\
&-\bar{C}_{M'}^{\pi}\left(h_{t}\right)-\sum_{c_t, s_{t+1}} P_{M'}\left(c_t, s_{t+1} \mid s_{t}, \pi(h_{t})\right) V_{t+1, M'}^{\pi}\left(h_{t+1}\right)
\end{align*}

\begin{align*}  
\left|V_{t, M}^{\pi}\left(h_{t}\right)-V_{t, M'}^{\pi}\left(h_{t}\right)\right|\leq \mid C_M^{\pi}&\left(h_{t}\right)-C_{M'}^{\pi}\left(h_{t}\right)\mid+\\ +\sum_{c_{t}, s_{t+1}}\mid &P_M\left(c_t, s_{t+1} \mid s_{t}, \pi(h_{t})\right) V_{t+1, M}^{\pi}\left(h_{t+1}\right)\\
- &P_{M}\left(c_t, s_{t+1} \mid s_{t}, \pi(h_{t})\right) V_{t+1, M'}^{\pi}\left(h_{t+1}\right)\\
+ &P_{M}\left(c_t, s_{t+1} \mid s_{t}, \pi(h_{t})\right) V_{t+1, M'}^{\pi}\left(h_{t+1}\right)\\
- &P_{M'}\left(c_t, s_{t+1} \mid s_{t}, \pi(h_{t})\right) V_{t+1, M'}^{\pi}\left(h_{t+1}\right)\mid
\end{align*}
So:
\begin{align*}
\mid V_{t, M}^{\pi}\left(h_{t}\right)-V_{t, M'}^{\pi}&\left(h_{t}\right)\mid \leq C_{max}\epsilon + \\ &+\sum_{c_t, s_{t+1}} P_M\left(c_t, s_{t+1} \mid s_{t}, \pi(h_{t})\right) \cdot \left| V_{t+1, M}^{\pi}\left(h_{t+1}\right)- V_{t+1, M'}^{\pi}\left(h_{t+1}\right)\right|\\
&+V_{t+1, M'}^{\pi}\left(h_{t+1}\right)\cdot \left|P_M\left(c_t, s_{t+1} \mid s_{t}, \pi(h_{t})\right)-P_{M'}\left(c_t, s_{t+1} \mid s_{t}, \pi(h_{t})\right)\right|
\end{align*}

% We know that $V^{\pi}_{t,M}(h_t) \leq C_{max}\cdot (T-t)$, so:
% \begin{align*}
% \left|V_{t, M}^{\pi}\left(h_{t}\right)-V_{t, M'}^{\pi}\left(h_{t}\right)\right|&\leq C_{max}\epsilon +\norm{V_{t+1, M}^{\pi}-V_{t+1, M'}^{\pi}}_\infty + C_{max}\cdot (T-t)\cdot \epsilon \\&= C_{max}\epsilon \cdot  \left(T-t+1\right)+ \norm{V_{t+1, M}^{\pi}-V_{t+1, M'}^{\pi}}_\infty
% \end{align*} 

We know that $V^{\pi}_{t+1,M}(h_{t+1}) \leq \left(T-1\right) C_{max}$, so:
\begin{align*}
\left|V_{t, M}^{\pi}\left(h_{t}\right)-V_{t, M'}^{\pi}\left(h_{t}\right)\right|&\leq C_{max}\epsilon +\norm{V_{t+1, M}^{\pi}-V_{t+1, M'}^{\pi}}_\infty + \left(T-1\right) C_{max} \epsilon \\&= TC_{max}\epsilon+ \norm{V_{t+1, M}^{\pi}-V_{t+1, M'}^{\pi}}_\infty
\end{align*} 

% Applying the above rule recurrently from $t=0$ to $t=T-2$: 
% \begin{align*}
% \left|V_{0, M}^{\pi}\left(h_{0}\right)-V_{0, M'}^{\pi}\left(h_{0}\right)\right|\leq C_{max}\epsilon + C_{max}\epsilon\cdot \sum_{t=3}^{T+1} t = C_{max}\epsilon \cdot \left(\frac{T^2 +3 T -2}{2}\right)\leq C_{max}\epsilon T^2
% \end{align*}

Applying the above rule recurrently from $t=0$ to $t=T-2$: 
\begin{align*}
\left|V_{0, M}^{\pi}\left(h_{0}\right)-V_{0, M'}^{\pi}\left(h_{0}\right)\right|\leq C_{max}\epsilon T^2
\end{align*}

Plugging the result for the loss difference:
$$
 \left| L_{M, \pi}-L_{M', \pi} \right| = \mu^T \left|V_{0, M}^{\pi}-V_{0, M'}^{\pi} \right|\leq C_{max}\epsilon T^2
$$
We receive the result by using assumption \ref{assumption:StateCostAssumption}.
\end{proof}

\subsection{PCA Error Bounds for Parametric Spaces with Low Dimensional Structure}\label{proof:decayingeigen}

\textbf{Lemma \ref{lemma:decayingeigen}}. Under Assumptions \ref{assumption:subgauss} and \ref{DecayingEigen}, we have that:
% \begin{small}
\begin{center}
$
\mathbb{E} \left[R\left(\widehat{P}_{d'}\right)\right] \leq \min \left(\frac{8C_{sg}^2 \sqrt{d'} tr(\Sigma)}{\sqrt{n}}, \frac{64 C_{sg}^4 tr^{2}(\Sigma)}{n\left(\lambda_{d'}-\lambda_{d'+1}\right)}\right) + \epsilon \cdot (d-d').
$
\end{center}

\begin{proof} [Proof of Lemma \ref{lemma:decayingeigen}]

A well known property of the PCA \cite{PCAbounds}:
$$
\min _{P \in \mathcal{P}_{d'}} R(P) = \sum_{i=d'+1}^d \lambda_i
$$
So:
\begin{equation*}
\begin{split}
\mathbb{E} R\left(\widehat{P}_{d'}\right) &\leq \left(\frac{8C_{sg}^2 \sqrt{d'} tr(\Sigma)}{\sqrt{n}}, \frac{64 C_{sg}^4 tr^{2}(\Sigma)}{n\left(\lambda_{d'}-\lambda_{d'+1}\right)}\right) + \sum_{i=d'+1}^d \lambda_i \\ &\leq \left(\frac{8C_{sg}^2 \sqrt{d'} tr(\Sigma)}{\sqrt{n}}, \frac{64 C_{sg}^4 tr^{2}(\Sigma)}{n\left(\lambda_{d'}-\lambda_{d'+1}\right)}\right) + \epsilon \cdot \left(d-d'\right)
\end{split}
\end{equation*}

Where we used Theorem \ref{theorem:PCABound} for the first inequality and assumption \ref{DecayingEigen} for the second.

\end{proof}

\subsection{Regret Bounds by Dimensionality Reduction} \label{proof:PCAKDEBounds}

\textbf{Theorem \ref{PCAKDEBounds}}. Let $\widehat{f}_{G}^{d'}$ be the approximation of $f$ as defined in Section \ref{sec:PCA}. Under Assumptions \ref{less than inf}, \ref{assumption:subgauss}, \ref{DecayingEigen}, and \ref{assumption:LowDimHolder} we have with probability at least $1-1/n$:

\begin{small}
$
\mathcal{R}_T\left(\pi^*_{\widehat{f}_{G}^{d'}}\right) \leq 2C_{max}T\left|\Theta_L\right| C_{d'} \left(\frac{\log{n}}{n}\right)^{\frac{\alpha'}{2\alpha'+d'}} + 2C_{max}T^2 C_g \sqrt{\min \left(\frac{C'_{sg} \sqrt{d'}}{\sqrt{n}}, \frac{C_{sg}^{\prime 2} }{n\Delta_{\lambda, d'}}\right) + \epsilon  (d-d')},
$

\end{small}
where 
\begin{small}
$C_{d'} = C_{\alpha'}2^{\frac{\alpha'-1}{2}} + \frac{16d' \sqrt{C_{\alpha'}\Delta_{max}^{\alpha'}\left(\Theta_L\right) +\frac{1}{\left| \Theta_L \right|}}}{\sqrt{2}\left(2\pi\right)^{\frac{d'}{4}}}+\frac{64d'^2}{(2\pi)^{\frac{d'}{2}}}$, $C'_{sg}=8C_{sg}^2 tr(\Sigma)$ and $\Delta_{\lambda, d'} = \lambda_{d'}-\lambda_{d'+1}$
\end{small}

\begin{proof} [Proof of Theorem \ref{PCAKDEBounds}]
\begin{equation*}
\begin{split}
     \mathcal{L}_{f_{\theta}}(\pi)&=
     \mathbb{E}_{\theta \sim f_{\theta}} L_{\theta, \pi}  =
    \int L_{\theta, \pi} f_{\theta}(\theta) d\theta =
    \int \left( L_{\theta, \pi}-L_{\widehat{P}_{d'}\cdot \theta, \pi}+L_{\widehat{P}_{d'}\cdot \theta, \pi}\right)f_{\theta}(\theta) d\theta \\&=
    \int \left( L_{\theta, \pi}-L_{\widehat{P}_{d'}\cdot \theta, \pi}\right)f_{\theta}(\theta) d\theta +\int L_{\widehat{P}_{d'}\cdot \theta, \pi} f_{\theta}(\theta) d\theta \\&=
    \int \left( L_{\theta, \pi}-L_{\widehat{P}_{d'}\cdot \theta, \pi}\right)f_{\theta}(\theta) d\theta +\int L_{\widehat{P}_{d'}\cdot \theta, \pi} f_{\theta}(\theta) d\theta  \\&\stackrel{(*)}{=}
    \int \left( L_{\theta, \pi}-L_{\widehat{P}_{d'}\cdot \theta, \pi}\right)f_{\theta}(\theta) d\theta +\int_{\theta_L\in \Theta_L}\left(\int_{\theta_L^{\perp} \in \Theta_L^{\perp}(\theta_L)} L_{\widehat{P}_{d'}\cdot \theta_L^{\perp}, \pi}f_\theta(\theta_L^{\perp})d\theta_L^{\perp}\right) d\theta_L  \\&\stackrel{(**)}{=}
    \int \left( L_{\theta, \pi}-L_{\widehat{P}_{d'}\cdot \theta, \pi}\right)f_{\theta}(\theta) d\theta +\int_{\theta_L\in \Theta_L}\left(L_{W_L^T\cdot \theta_L, \pi}\int_{\theta_L^{\perp} \in \Theta_L^{\perp}(\theta_L)} f_\theta(\theta_L^{\perp})d\theta_L^{\perp}\right) d\theta_L  \\&=
    \int \left( L_{\theta, \pi}-L_{\widehat{P}_{d'}\cdot \theta, \pi}\right)f_{\theta}(\theta) d\theta +\int L_{W_L^T\cdot \theta_L, \pi} f_{\theta_L}(\theta_L) d\theta_L 
\end{split}
\end{equation*}

$(*)$ Fubini theorem holds since $\int L_{\widehat{P}_{d'}\cdot \theta, \pi} f_{\theta}(\theta) d\theta \leq \infty$

$(**)$ Since $W_L^T \cdot \theta_L = \widehat{P}_{d'}\cdot \theta_L^{\perp}$, $\forall \theta_L^{\perp} \in \Theta_L^{\perp}(\theta_L)$.

And we know that:

\begin{equation*}
\begin{split}
\mathcal{L}_{\widehat{f}_{G}^{d'}}(\pi) = \mathbb{E}_{\theta \sim \widehat{f}_{G}^{d'}}L_{\theta, \pi}  = \int L_{W_L^T\cdot \theta_L, \pi} \widehat{f}_G(\theta_L) d\theta_L
\end{split}
\end{equation*}

Subtracting the two and taking the absolute value and using the triangle inequality:
\begin{equation*}
\begin{split}
\left|\mathcal{L}_{f_{\theta}}(\pi)-\mathcal{L}_{\widehat{f}_{G}^{d'}}(\pi) \right| &\leq \int \left| L_{\theta, \pi}-L_{\widehat{P}_{d'}\cdot \theta, \pi}\right| f_{\theta}(\theta) d\theta +\int L_{W_L^T\cdot \theta_L, \pi} \left | f_{\theta_L}(\theta_L)-\widehat{f}_G(\theta_L)\right| d\theta_L
\end{split}
\end{equation*}

Starting with the first term:
\begin{equation*}
\begin{split}
\int \left| L_{\theta, \pi}-L_{\widehat{P}_{d'}\cdot \theta, \pi}\right| f_{\theta}(\theta) d\theta &\stackrel{(*)}{=}
C_{max}C_g T^2 \cdot \int \left| \theta - \widehat{P}_{d'}\cdot \theta\right| f_{\theta}(\theta) d\theta \\&\stackrel{(**)}{=} C_{max}C_g T^2 \cdot \sqrt{\int \left( \theta - \widehat{P}_{d'}\cdot \theta\right)^2 f_{\theta}(\theta) d\theta}\\
&= C_{max}C_g T^2 \cdot \sqrt{\mathbb{E} R\left(\widehat{P}_{d'}\right)} \\
&\stackrel{(***)}{\leq} \sqrt{\min \left(\frac{8C_{sg}^2 \sqrt{d'} tr(\Sigma)}{\sqrt{n}}, \frac{64 C_{sg}^4 tr^{2}(\Sigma)}{n\left(\lambda_{d'}-\lambda_{d'+1}\right)}\right) + \epsilon \cdot (d-d')}
\end{split}
\end{equation*}

$(*)$ Using Lemma \ref{lemma:LipCont}

$(**)$ Using the fact that $\sqrt{x}$ is concave, and using Jensen inequality ($\mathbb{E}[\sqrt{x}]\leq \sqrt{\mathbb{E}[x]}$):

$(***)$ Using Lemma \ref{lemma:decayingeigen}

The second term can be bounded using lemma $\ref{lemma:gaussianKDEbounds}$:
\begin{equation*}
\begin{split}
\int L_{W_L^T\cdot \theta_L, \pi} \left | f_{\theta_L}(\theta_L)-\widehat{f}_G(\theta_L)\right| d\theta_L \leq \left|\Theta_L\right|  C_{d'} C_{max}T \left(\frac{\log{n}}{n}\right)^{\frac{\alpha}{2\alpha+d'}}
\end{split}
\end{equation*}

Adding the two terms and using the same argument as in the proof from Section \ref{proof:L1_BO_bound} gives the result. 
\end{proof}

% Taking the absolute value and using the triangle inequality:
% \begin{equation*}
% \begin{split}
% \left|\mathcal{L}_{f_x}(\pi)-\mathcal{L}_{\hat{f_z}}(\pi) \right| &\leq \int \left| L_{x, \pi}-L_{\hat{P}_{d'}\cdot x, \pi}\right| f_x(x) dx +\int L_{W_L^T\cdot z, \pi} \left | f_z(z)-\widehat{f}_G(z)\right| dz \\&\leq
% C\cdot \int \left| x - \hat{P}_{d'}\cdot x\right| f_x(x) dx +\int L_{W_L^T\cdot z, \pi} \left | f_z(z)-\widehat{f}_G(z)\right| dz
% \end{split}
% \end{equation*}

% Using the fact that $\sqrt{x}$ is concave, and using Jensen inequality ($\mathbb{E}[\sqrt{x}]\leq \sqrt{\mathbb{E}[x]}$):
% \begin{equation*}
% \begin{split}
% \left|\mathcal{L}_{f_x}(\pi)-\mathcal{L}_{\hat{f_z}}(\pi) \right| &\leq C\cdot \sqrt{\int \left( x - \hat{P}_{d'}\cdot x\right)^2 f_x(x) dx} +\int L_{W_L^T\cdot z, \pi} \left | f_z(z)-\widehat{f}_G(z)\right| dz
% \end{split}
% \end{equation*}

% \end{proof}

\newpage

\subsection{VariBad Dream Implementation Details}\label{sec:implementation}
In our implementation (which is formulated in Algorithm \ref{alg:VariBadDream}), we first run the regular VariBad training scheme for $I_W$ warm-up iterations (because the dream environments at the very start of the training are uninformative). After the warm-up period, at each iteration we insert the last encoded latent vector from each of the real environments (i.e after $H$ steps) into a latent pool. Every $I_{KDE}$ iterations we updated the KDE estimation. At each iteration we sample $n_{dream}$ vectors from the KDE and pass one to each dream environment worker. Each dream environment will use this latent vector and the reward decoder to assign rewards. 
\begin{algorithm}
\caption{VariBad Dream}\label{alg:VariBadDream}
\begin{algorithmic}

\Require $\{M_i\}_{i=1}^N\in \M^N \text{  The training MDPs}$, \par
\hskip\algorithmicindent $R, D \in \mathbb{N}$ Number of real and dream agents respectively \par
\hskip\algorithmicindent $I_W, I_T \in \mathbb{N}$ Number of warmup and training iterations respectively \par
\hskip\algorithmicindent $I_{KDE}\in \mathbb{N}$ KDE update interval

\State $\textit{real\_workers} \gets \{real\_worker()\}_{i=1}^R$
\State $\textit{dream\_workers} \gets \{dream\_worker()\}_{i=1}^D$
\State $\textit{latents\_pool} \gets \{\}$
\For{$i \gets 0  \textit{  to } I_W$}
\State  $\textit{latents\_pool} \gets \textit{latents\_pool} \cup \textit{real\_workers.run\_episode()}$  \Comment{Warmup iterations}
\EndFor
\For{$i \gets 0  \textit{  to } I_T$}
\If{$i \mod I_{KDE} = 0 $}
\State $dream\_workers.kde \gets \textit{kde(latents\_pool)}$  
\State $\textit{latents\_pool} \gets \{\}$
\EndIf
\State  $\textit{latents\_pool} \gets \textit{latents\_pool} \cup \textit{real\_workers.run\_episode()}$  \Comment{Main iterations}
\State  $\textit{dream\_workers.run\_episode()}$ 
\State $\textit{VariBad.vae\_update()}$ \Comment{Original VAE update}
\State $\textit{VariBad.policy\_update()}$ \Comment{Original policy update}
\EndFor
\Function{$\textit{real\_worker.run\_episode}$}{}
  \State $posterior\_latents \gets \{\}$
  \For{$\textit{real\_worker in real\_workers}$}
    \State $real\_worker.mdp = random\_sample(\{M_i\}_{i=1}^N)$
    \State $\textit{VariBad\_rollout(real\_worker.mdp)}$ \Comment{Steps in the environment and buffers updates}
    \State $\textit{posterior\_latents} \gets \textit{posterior\_latents} \cup \textit{vae.posterior}$
  \EndFor
  \State $\textit{return posterior\_latents}$
\EndFunction
\Function{$\textit{dream\_worker.run\_episode}$}{}
\For{$\textit{dream\_worker in dream\_workers}$}
\State $\textit{dream\_worker.latent = dream\_workers.kde.sample()}$
\State $\textit{curr\_mdp} \gets vae.decoder(dream\_worker.latent)$  
\State \Comment{MDP's transitions and reward defined by the decoder's outputs}
\State $\textit{VariBad\_rollout(curr\_mdp)}$ \Comment{Original episode rollout}
\EndFor
\EndFunction

\end{algorithmic}
\end{algorithm}

Our implementation is based on the open-source code of Zintgraf et al \cite{VariBad}, which can be found in \url{https://github.com/lmzintgraf/varibad}.

The code implementing VariBad Dream,  and the details on how to reproduce all the experiments presented in this paper, can be found in \url{https://github.com/zoharri/MBRL2}.
% The code implementing VariBad Dream, and the details on how to reproduce all the experiments
% presented in this paper, are in the supplementary materials.

Hyperparameters for VariBad:
\begin{center}
	\begin{tabular}{ |l|l| } 
		\hline
		Rollout horizon & 100 \\
		Number of rollouts & 2 \\
		\hline
		RL algorithm & PPO \\ 
		Epochs & 2 \\ 
		Minibatches & 4 \\ 
		Max grad norm & 0.5 \\ 
		Clip parameter & 0.05  \\ 
		Value loss coeff. & 0.5 \\ 
		Entropy coeff. & 0.01  \\
		Gamma & 0.97 \\
		\hline
		Weight of KL term in ELBO & 1 \\ 
		Policy learning rate & 7e-4 \\ 
		VAE learning rate & 1e-3  \\
		VAE batch size & 5 \\ 
		Task embedding size & 5 \\ 
		\hline
		Policy architecture & 2 hidden layers, 128-dim each, TanH activations  \\ 
		\hline
		Encoder architecture & States/actions/rewards encoder: FC layer 32/16/16 dim, \\ & GRU with hidden size 128 , \\ & output layer of dim 5, ReLu activations \\
		\hline 
		Reward decoder architecture & 2 hidden layers, 64 and 32 dims, \\ & ReLu activations \\
		Reward decoder loss function & Mean squared error\\
		\hline
	\end{tabular}
\end{center}
The Hyper parameters for VariBad Dream are the same as for VariBad with the following additional parameters:
\begin{center}
	\begin{tabular}{ |l|l| } 
		\hline
		Number of warm-up iterations  & 5000  \\ 
		KDE update interval & 3 \\
		\hline
	\end{tabular}
\end{center}
For the case of 20/30 sample sizes we chose 4/6 dream environment workers and 12/10 real environment workers. The intuition is that as we have more training environments, the better the latent representation and KDE are, and we can rely more on the dream environments. 

Similarly to \cite{VariBad} and \cite{LDM}, we used only the reward decoder (and not the state decoder) due to better empirical results. 

\subsection{KDE and Mixup Dream Environments Comparison} \label{sec:mixupComp}
In this section we compare between using the proposed KDE and the simple Mixup approach to generate dream environments.
While Mixup approaches usually aim at solving out of distribution generalization, we use it here as a baseline for the prior estimation. We emphasize that we don't claim to solve OOD generalization.    
In the Mixup approach, given $N$ latent vectors of real environments, we sample $N$ coefficients from the Dirichlet distribution $\vec{\alpha} = Dir\left(1, \dots, 1\right)$ and use them to calculate a weighted average for the dream environment.

\begin{figure}
     \centering
     \vspace{-1em}
     \begin{subfigure}[l]{0.45\textwidth}
         \centering
         \includegraphics[width=\textwidth]{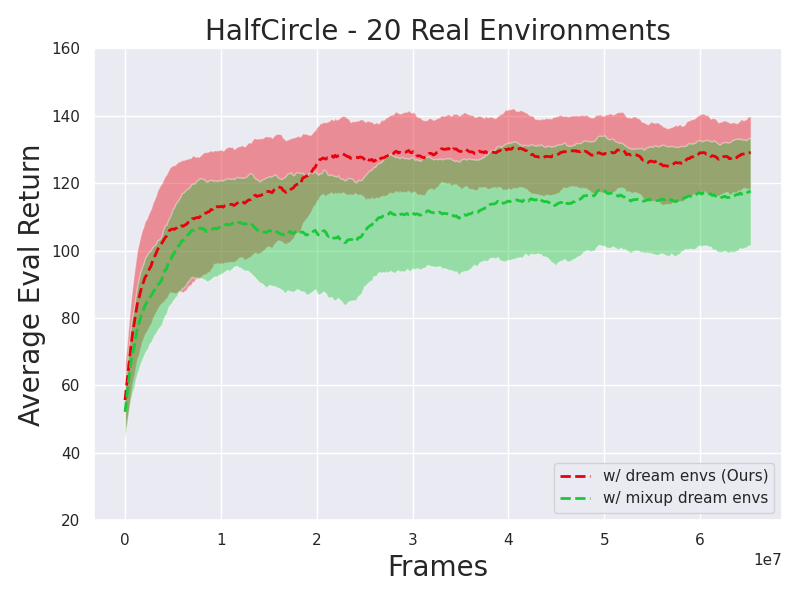}
     \end{subfigure}
     \hspace{-1em}
     \begin{subfigure}[r]{0.45\textwidth}
         \centering
         \includegraphics[width=\textwidth]{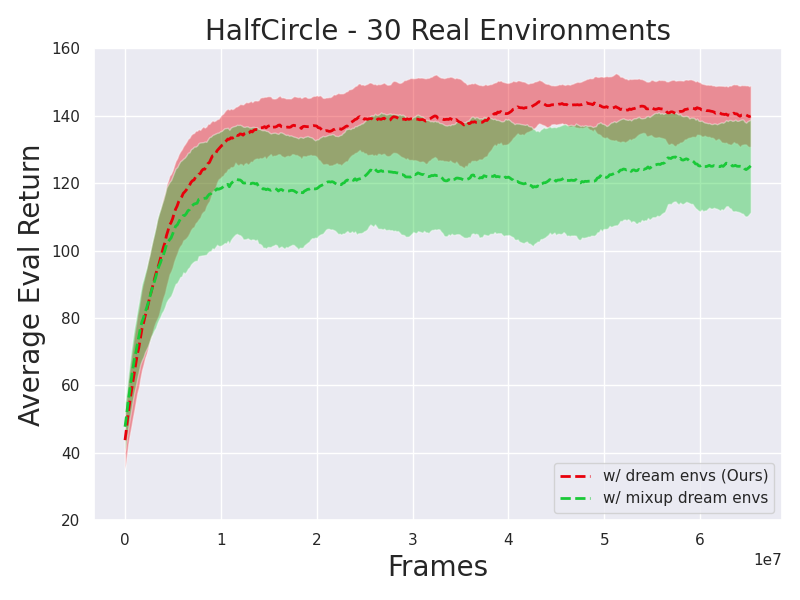}
     \end{subfigure}
     \vspace{-1em}
     \caption{Average return on HalfCircle with KDE and mixup dream environments. The average is shown in dashed lines, with the 95\% confidence intervals (15 random seeds).}
     \label{figure:HalfCircleMixupComp}
\end{figure}

In Figure \ref{figure:HalfCircleMixupComp} we compare the average evaluation return of the proposed KDE dream environments to the Mixup dream environments. Evidently, our approach achieved higher return, both for $N_{train}=20$ and $N_{train}=30$.

\begin{figure}
    \centering
    \setlength{\belowcaptionskip}{-6pt}
    \setlength{\tabcolsep}{1.5pt}
    {\small
    \begin{tabular}{c c  c c c c c }
        \vspace{-0.0615cm}
        % KDE
        \raisebox{0.18in}{\rotatebox{90}{\textbf{KDE}}} &
        \includegraphics[width=0.15\textwidth]{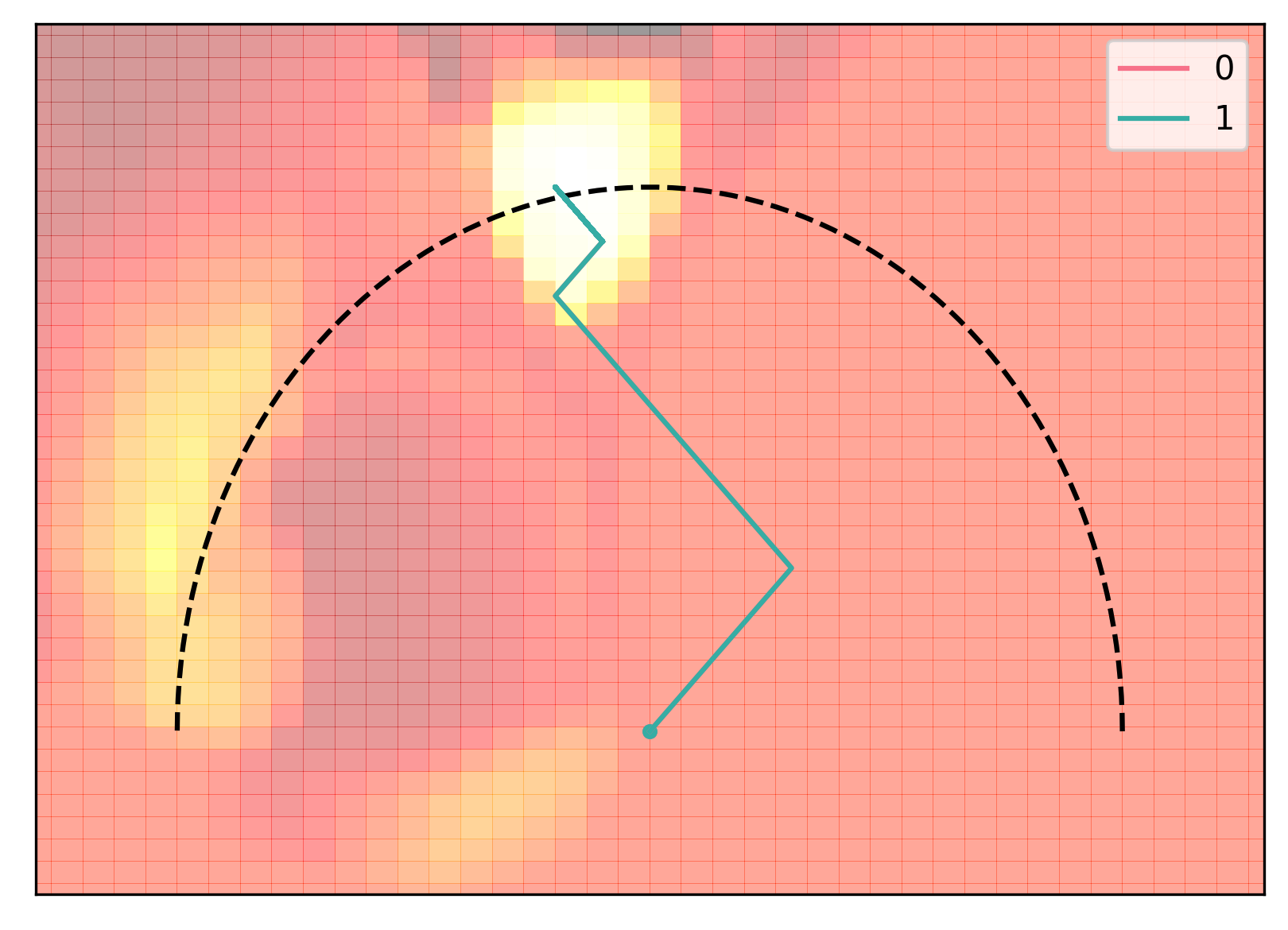} &
        \includegraphics[width=0.15\textwidth]{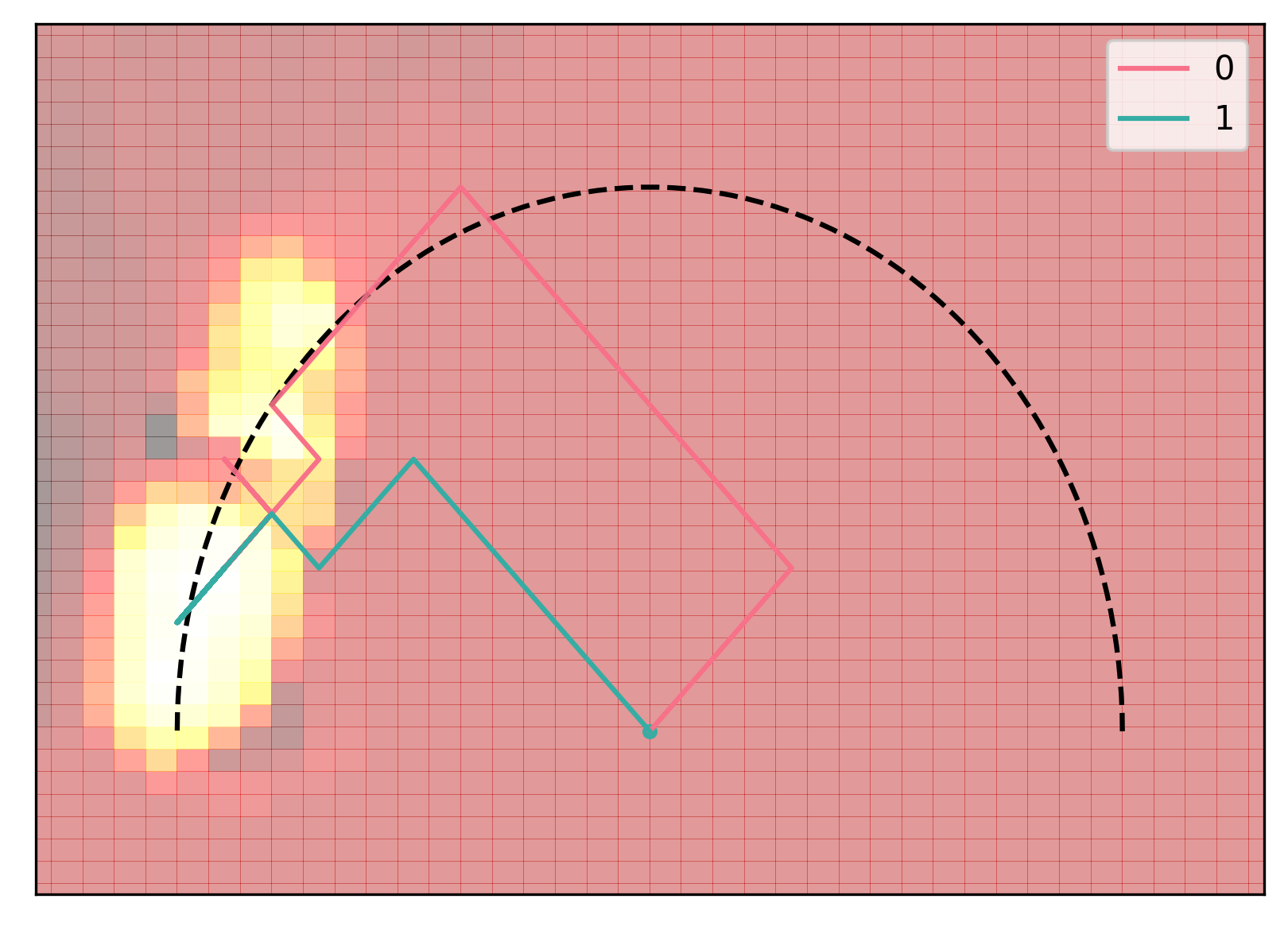} &
        \includegraphics[width=0.15\textwidth]{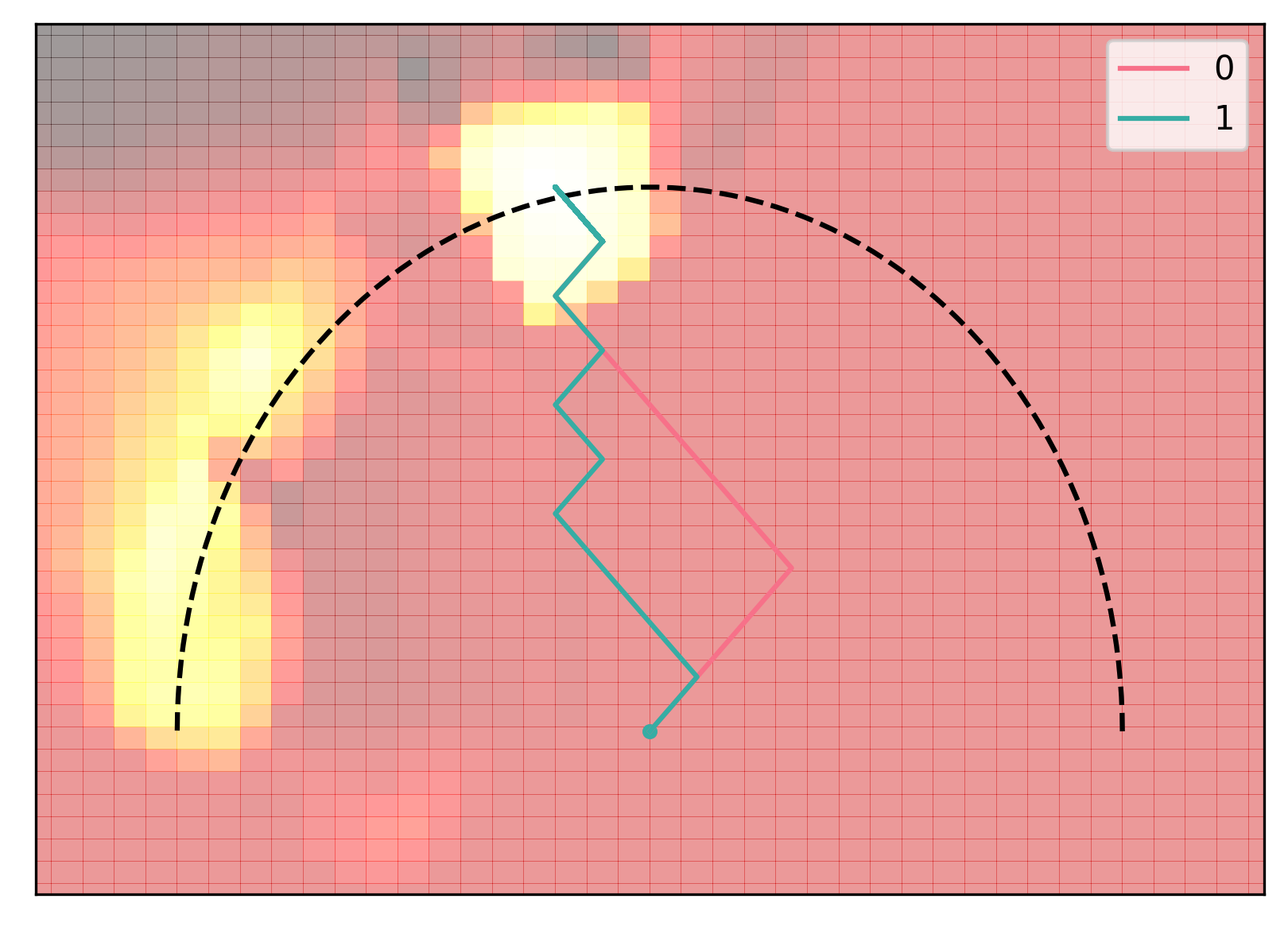} &
        \includegraphics[width=0.15\textwidth]{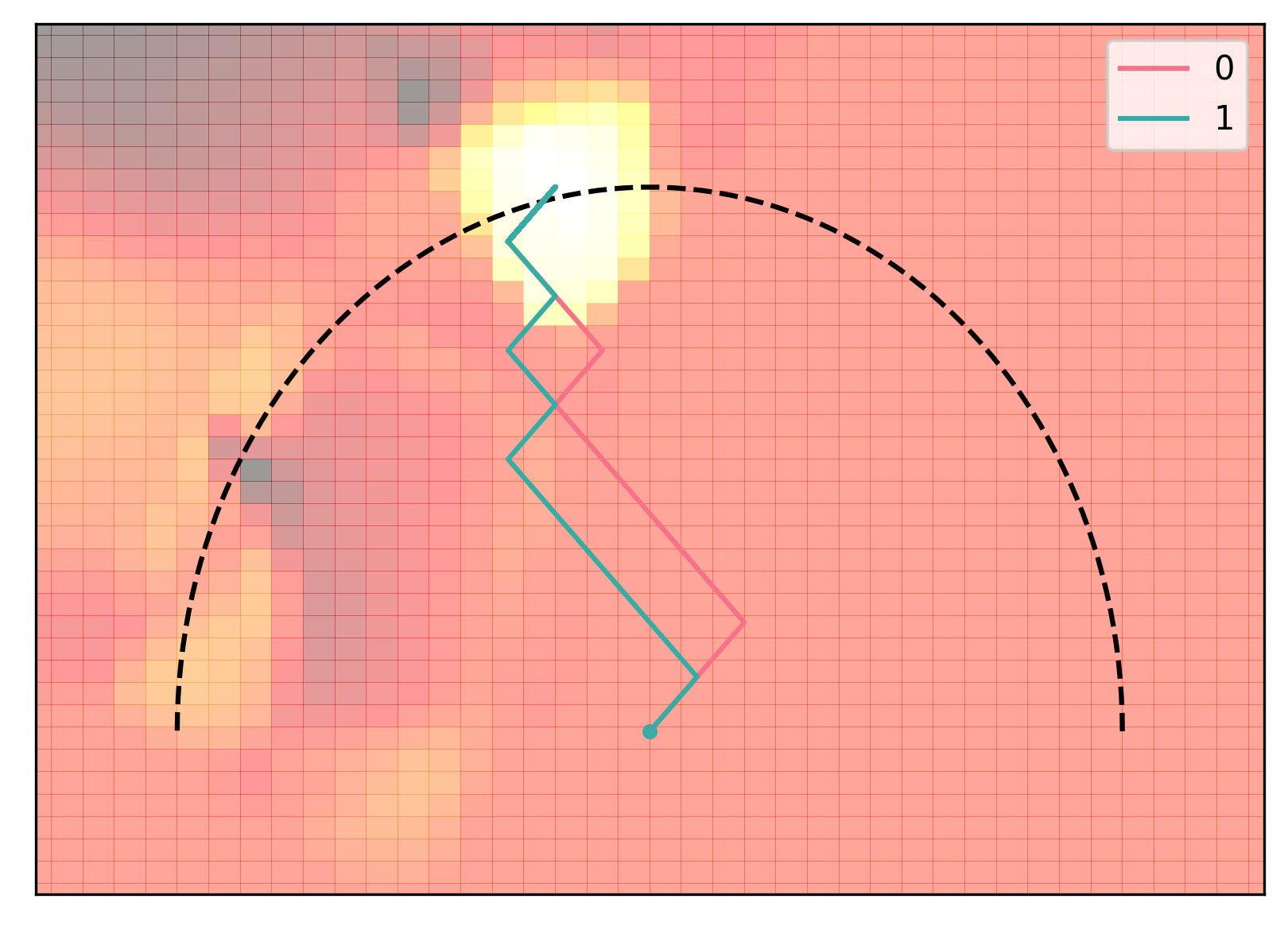} &
        \includegraphics[width=0.15\textwidth]{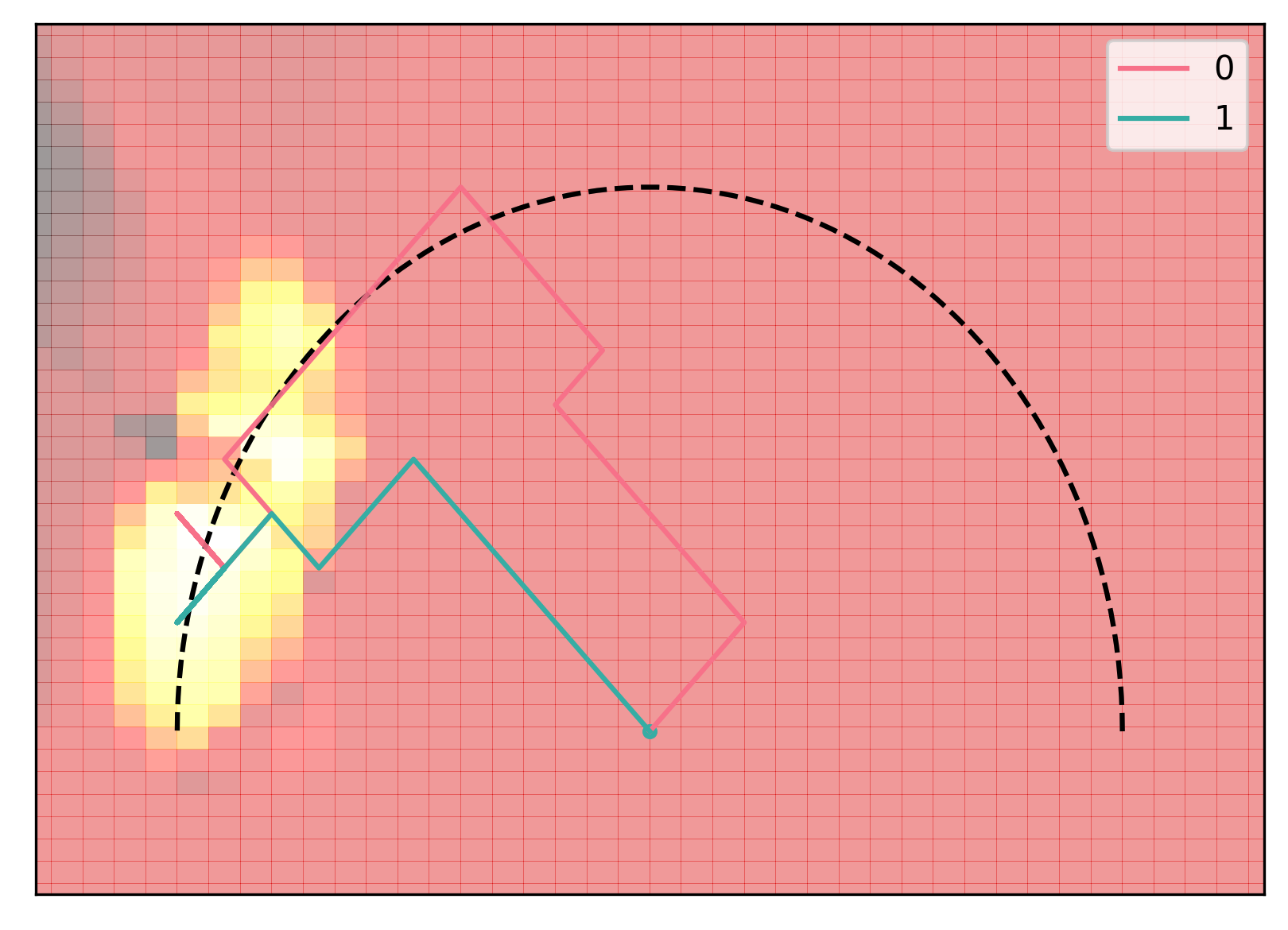} &
        \includegraphics[width=0.15\textwidth]{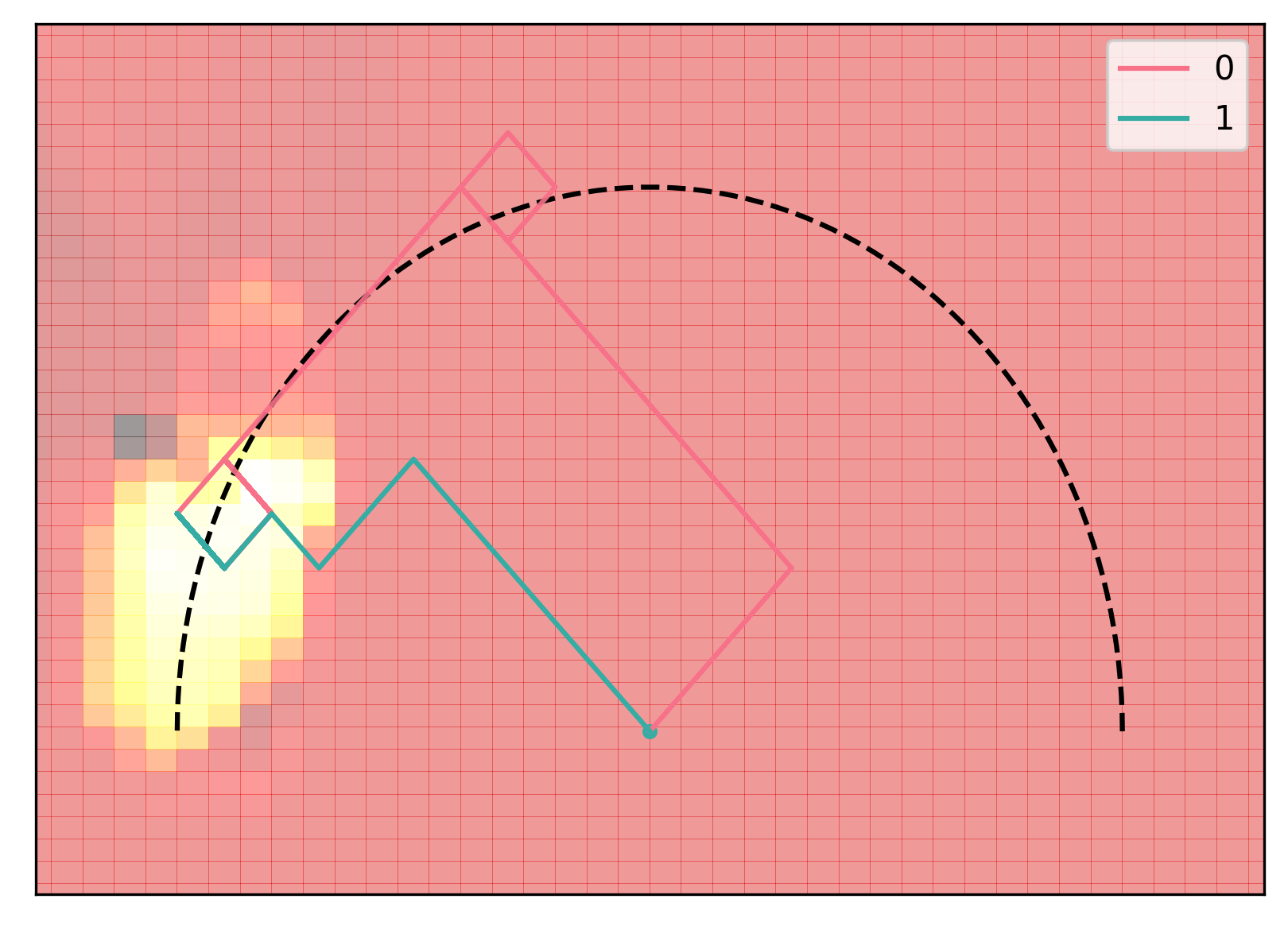}  \\

        % Mixup
        \raisebox{0.13in}{\rotatebox{90}{\textbf{Mixup}}} &
        \includegraphics[width=0.15\textwidth]{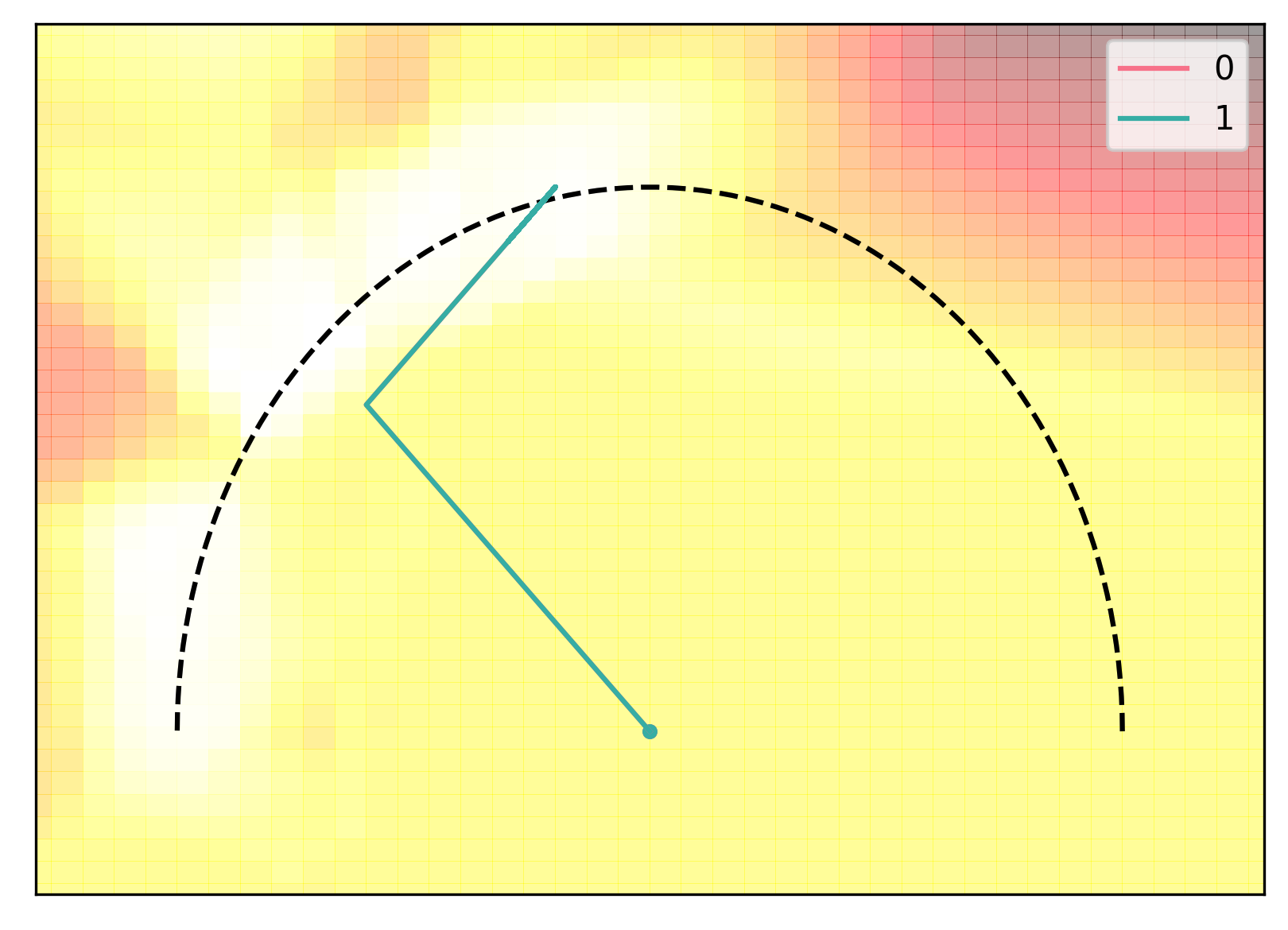} &
        \includegraphics[width=0.15\textwidth]{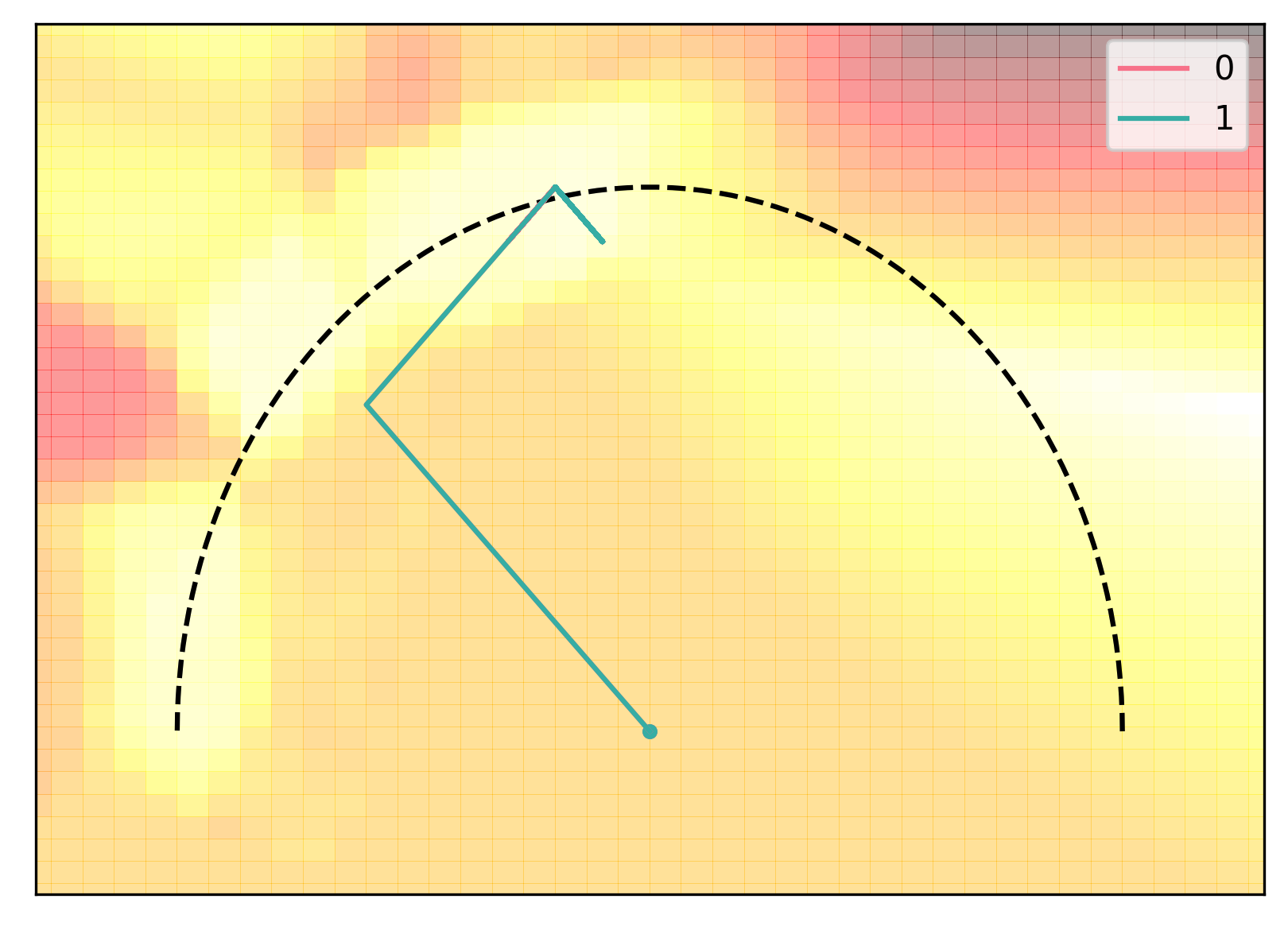} &
        \includegraphics[width=0.15\textwidth]{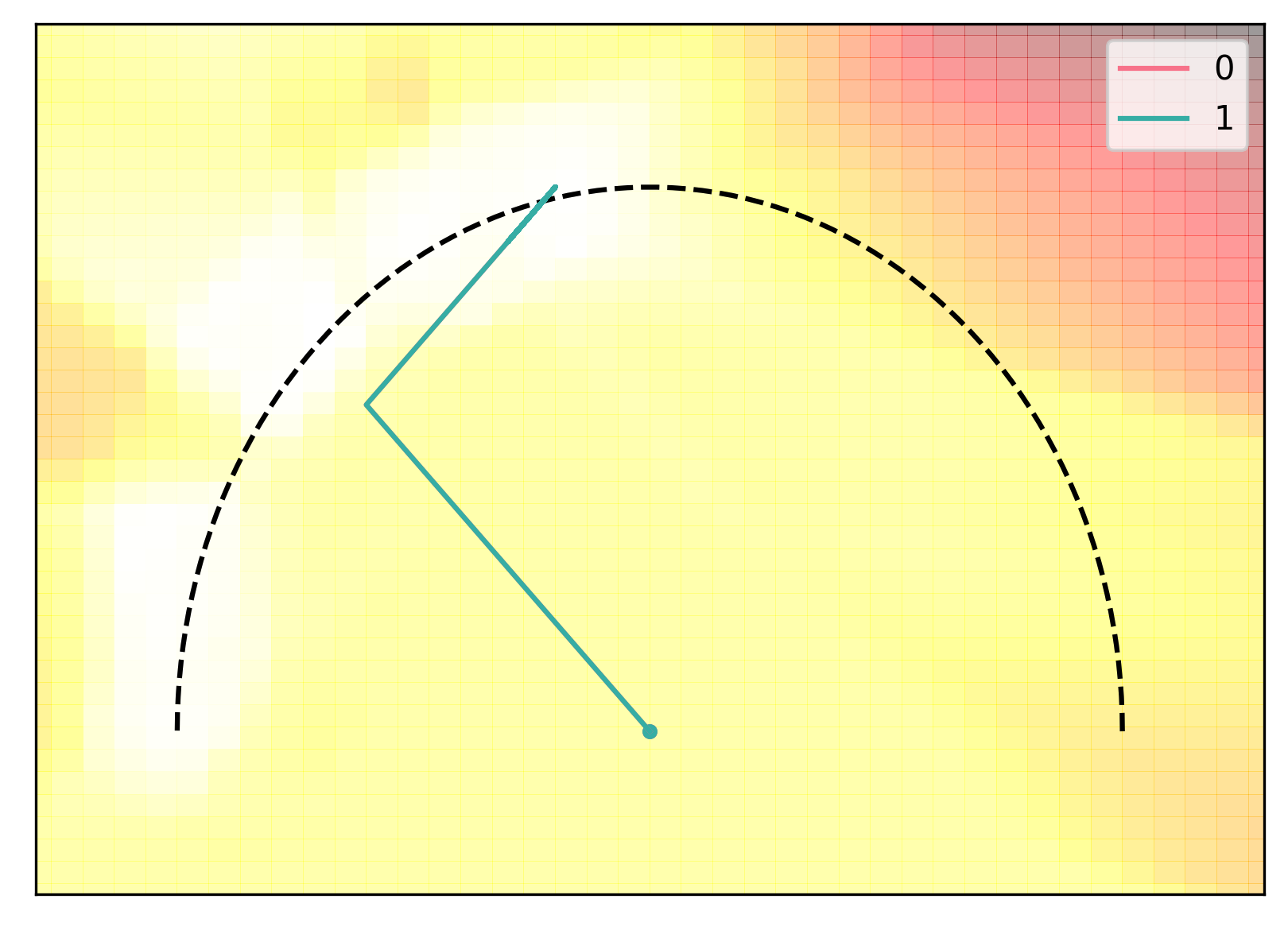} &
        \includegraphics[width=0.15\textwidth]{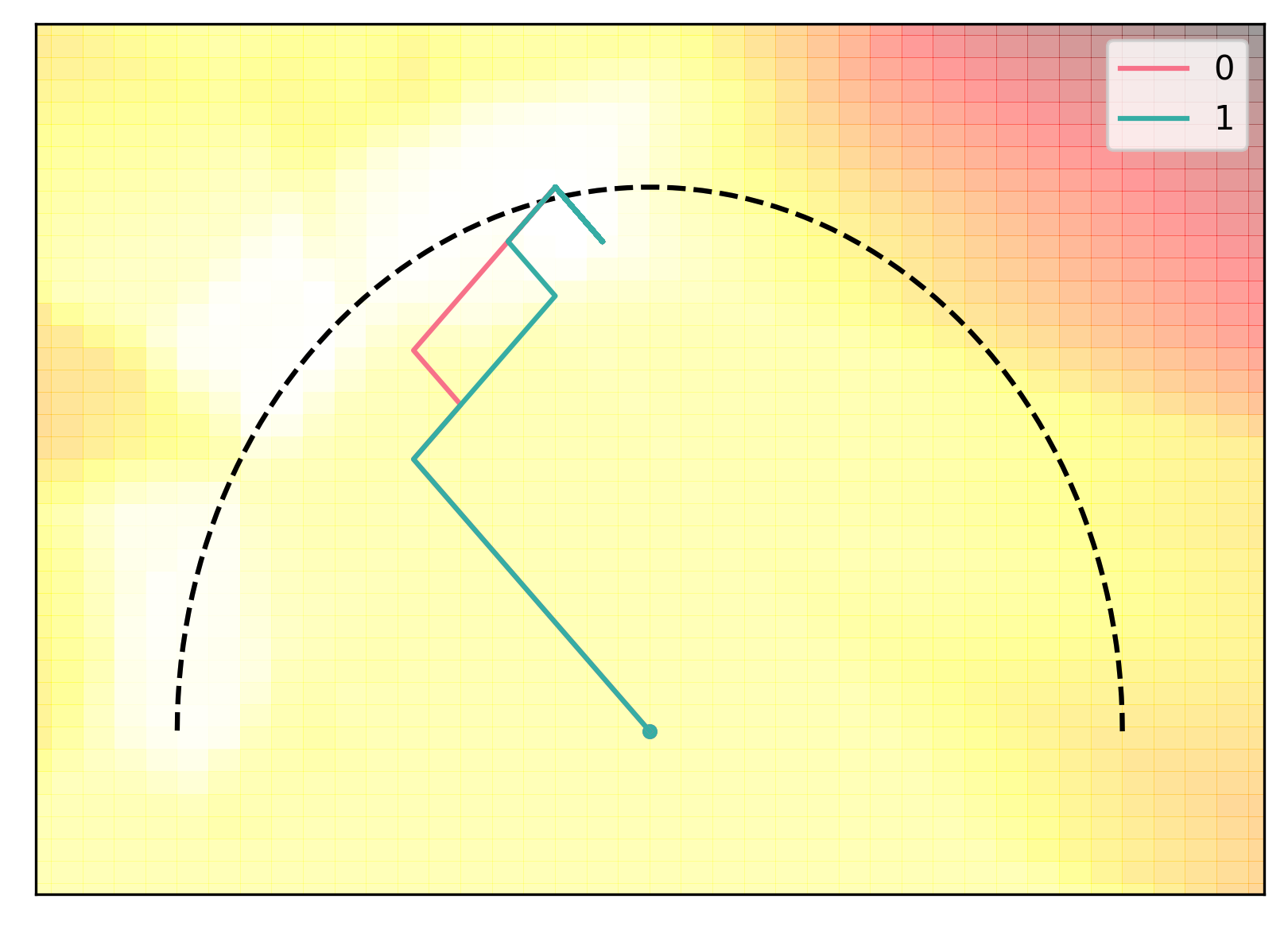} &
        \includegraphics[width=0.15\textwidth]{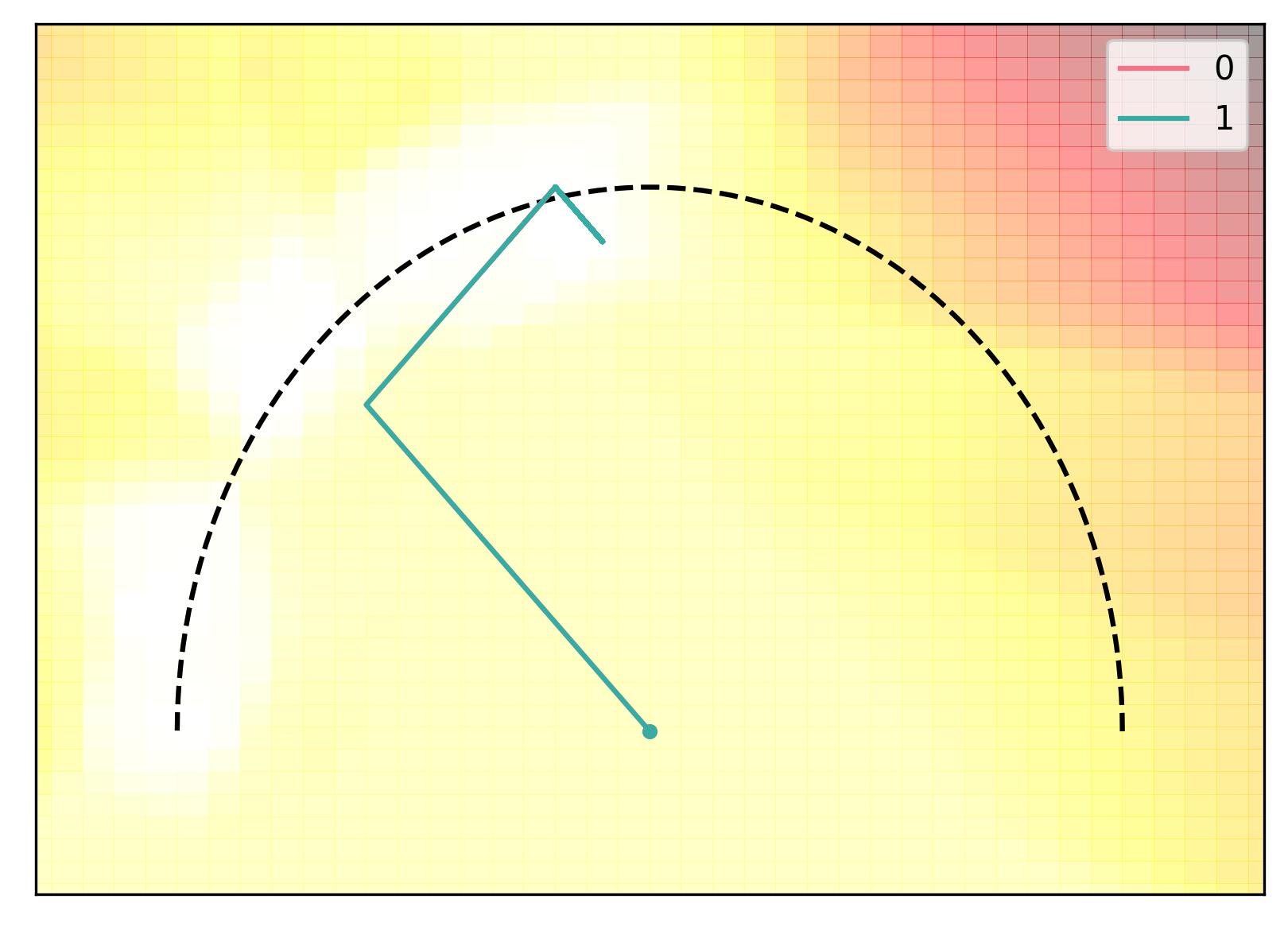} &
        \includegraphics[width=0.15\textwidth]{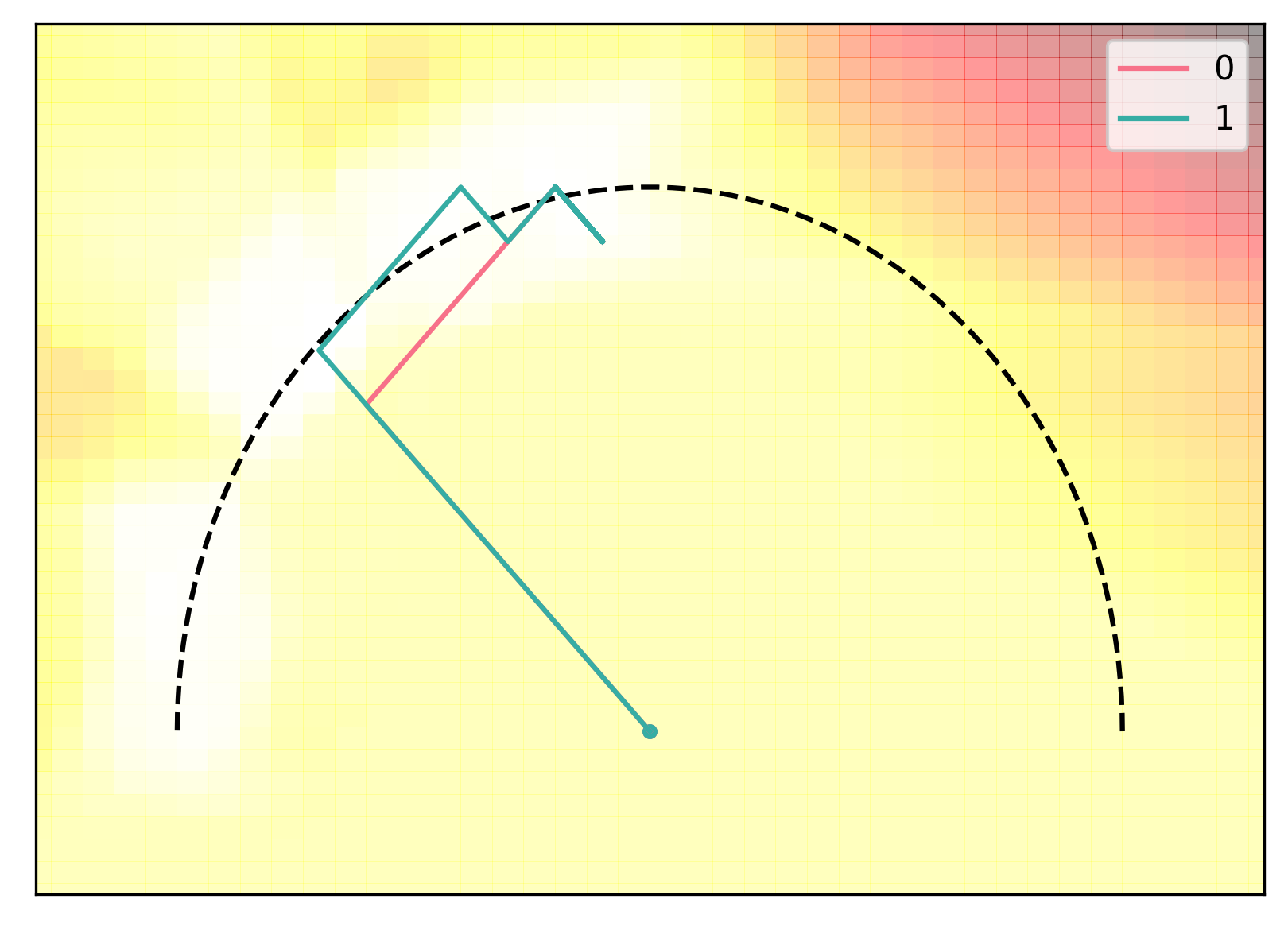}\\
        
    \end{tabular}
    
    }
    \vspace{-0.225cm}
    \caption{Comparison of dream environments' reward maps, generated during the training using 20 real environments. On the top row - our suggested KDE method. On the bottom row - the Mixup method. Each column corresponds to a different training iteration with a 1000 iteration interval between each one. The trajectory of the policy for the sampled dream environment is plotted on top of the reward map as well.}
    \vspace{-0.225cm}
    \label{fig:dreamvis}
\end{figure}

To further analyze the differences between the two approaches and to better understand the dream environments, we visualize the reward map of the dream environments, as can be seen in Figure \ref{fig:dreamvis}. In order to visualize the reward map for a given latent vector (sampled using either KDE or Mixup), we pass a discrete grid of states and the latent vector to the reward decoder and draw a heatmap of the results. In Figure \ref{fig:dreamvis} we plot multiple sampled latent vectors, which were observed during the training. We can see that the KDE produces much more realistic and variable dream environments than the Mixup approach, explaining its superior performances. 

\subsection{Ant Goal Environment} \label{sec:ant_goal_env}
\begin{figure}[H]
     \centering
     \vspace{-1em}
     \begin{subfigure}[l]{0.6\textwidth}
         \centering
         \includegraphics[width=\textwidth]{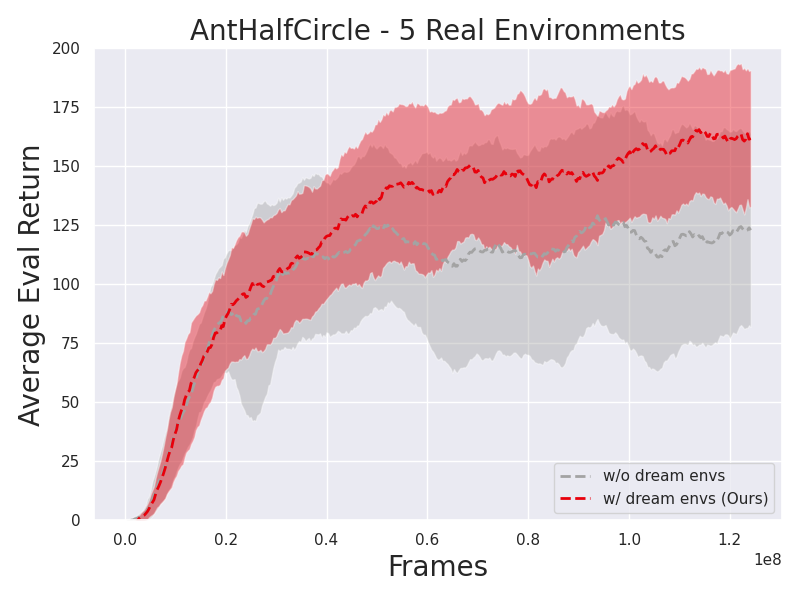}
     \end{subfigure}
     \vspace{-1em}
     \caption{ Average return on the Ant Goal environment with and without dream environments. The average is shown in dashed lines, with the 95\% confidence interval (8 random seeds).}
     \label{figure:AntGoal}
\end{figure}
Our PAC bounds showed that the determining factor in generalization is not the  dimensionality of the MDP (i.e., the dimension of the state and action spaces), but the dimensionality of the underlying MDP parameter space $\Theta$. 
% should not affect generalization. We showed that only the underlying changing parameters of the MDPs are relevant. 

In this section, we demonstrate this claim by running VariBad Dream on a high-dimensional continuous control problem -- a variant of the HalfCircle environment (Figure \ref{halfcircle}) with an Ant robot agent, simulated in MuJoCo \cite{todorov2012mujoco}. Note that while the space $\Theta$ is low dimensional, similarly to the point robot HalfCircle experiments, each MDP has a high dimensional state and action space.

In order to make the baseline VariBAD method work on this environment properly, we increased the goal size from 0.2 (which was used in the HalfCircle environment) to 0.3. We found that this change was necessary for VariBAD to explore effectively and reach the goal during training.

In Figure \ref{figure:AntGoal} we compare VariBAD and VariBAD Dream on this environment.
% show the effect of incorporating the dream environments on the average return for the evaluation environments. 
Similarly to the point robot results, we observe an advantage for VariBAD Dream when the number of training MDPs is small ($N_{train}=5$). For $N_{train}=10$ we did not observe significant advantage for VariBAD Dream, as, similarly to the $N_{train}=40$ case in point robot, the training samples already cover the space $\Theta$ adequately in most runs.
% our approach achieves better results. Since we enlarged the goal radius to make the original VariBad algorithm work, more than $N_{train}=5$ was enough to cover all the goal space.

\subsection{Compute Specification}\label{sec:compute}
We ran the experiments on a Standard\_NC24s\_v3 Azure machine consisting of 4 NVIDIA Tesla V100 GPUs. Running 15 seeds at a time took a total of $\sim$24 hours.

\end{document}